\documentclass[10pt,journal,compsoc]{IEEEtran}
\usepackage{cite}
\usepackage{amsmath}
\usepackage{amsfonts}
\usepackage{color}
\usepackage{comment}
\usepackage{subfigure}
\usepackage{graphicx}
\usepackage{multirow}
\usepackage{algorithm}
\usepackage{algorithmic}
\usepackage{mathtools}
\usepackage{bm}
\usepackage{url}
\usepackage{booktabs}  
\usepackage{bbding}
\usepackage[switch]{lineno}

\newtheorem{theorem}{Theorem}

\newtheorem{Lemma}{Lemma}
\newtheorem{assumption}{Assumption}

\newtheorem{proof}{Proof}
\newtheorem{mydef}{Definition}

\begin{document}
%
\title{Model Inversion Attacks against \\ Graph Neural Networks}
%
%
%
%

\author{Zaixi Zhang,
        Qi Liu, Zhenya Huang, Hao Wang, Chee-Kong Lee,
        and Enhong Chen
\IEEEcompsocitemizethanks{\IEEEcompsocthanksitem Zaixi Zhang, Qi Liu, Zhenya Huang, Hao Wang, and Enhong Chen are with the Anhui Province Key Laboratory of BIg Data Analysis and Application (BDAA), School of Computer Science and Technology, University of Science and Technology of China, Hefei, Anhui 230027, China.\protect\\
E-mail: $\{$zaixi, haowang3$\}$@mail.ustc.edu.cn, $\{$qiliuql, huangzhy, cheneh$\}$@ustc.edu.cn
\IEEEcompsocthanksitem Chee-Kong Lee is with Tencent Quantum Lab. \protect\\
E-mail: cheekonglee@tencent.com}
\thanks{Manuscript received February 24, 2022; accepted September 9, 2022.\\Qi Liu is the corresponding author.}}

\markboth{IEEE TRANSACTIONS ON KNOWLEDGE AND DATA ENGINEERING}%
{Shell \MakeLowercase{\textit{et al.}}: Bare Advanced Demo of IEEEtran.cls for IEEE Computer Society Journals}

\IEEEtitleabstractindextext{%
\begin{abstract}
Many data mining tasks rely on graphs to model relational structures among individuals (nodes). Since relational data are often sensitive, there is an urgent need to evaluate the privacy risks in graph data. 
One famous privacy attack against data analysis models is the model inversion attack, which aims to infer sensitive data in the training dataset and leads to great privacy concerns. Despite its success in grid-like domains, directly applying model inversion attacks on non-grid domains such as graph leads to poor attack performance. This is mainly due to the failure to consider the unique properties of graphs. To bridge this gap, we conduct a systematic study on model inversion attacks against Graph Neural Networks (GNNs), one of the state-of-the-art graph analysis tools in this paper. 
Firstly, in the white-box setting where the attacker has full access to the target GNN model, we present GraphMI to infer the private training graph data.
Specifically in GraphMI, a projected gradient module is proposed to tackle the discreteness of graph edges and preserve the sparsity and smoothness of graph features; a graph auto-encoder module is used to efficiently exploit graph topology, node attributes, and target model parameters for edge inference; a random sampling module can finally sample discrete edges. Furthermore, in the hard-label black-box setting where the attacker can only query the GNN API and receive the classification results, we propose two methods based on gradient estimation and reinforcement learning (RL-GraphMI).
With the proposed methods, we study the connection between model inversion risk and edge influence and show that edges with greater influence are more likely to be recovered.
Extensive experiments over several public datasets demonstrate the effectiveness of our methods.
We also evaluate our attacks under two defenses: one is the well-designed differential private training, and the other is graph preprocessing. Our experimental results show that such defenses are not sufficiently effective and call for more advanced defenses against privacy attacks.
\end{abstract}

}
\maketitle

\IEEEdisplaynontitleabstractindextext

%
\vspace{-0.6cm}
\section{Introduction}

\IEEEPARstart{G}{raph} structured data have enabled a series successful applications such as as recommendation systems \cite{zhang2020personalized}, social relationship analysis \cite{fan2020graph}, and molecule property prediction\cite{zhang2021motif}, given the rich node attributes and edges information. As one of the state-of-the-art graph analysis tools, Graph Neural Networks (GNNs) \cite{kipf2016semi,velivckovic2017graph, hamilton2017inductive} has shown excellent performances in various tasks on graph. Typically, GNNs follow a message passing paradigm to aggregate the information from neighboring nodes based on edge connections, update node representations, and obtain the embeddings for downstream tasks such as node/graph classification and link prediction \cite{pmlr-v70-gilmer17a}. 

\begin{figure}[!t]
	\centering
	\includegraphics[width=3.0 in]{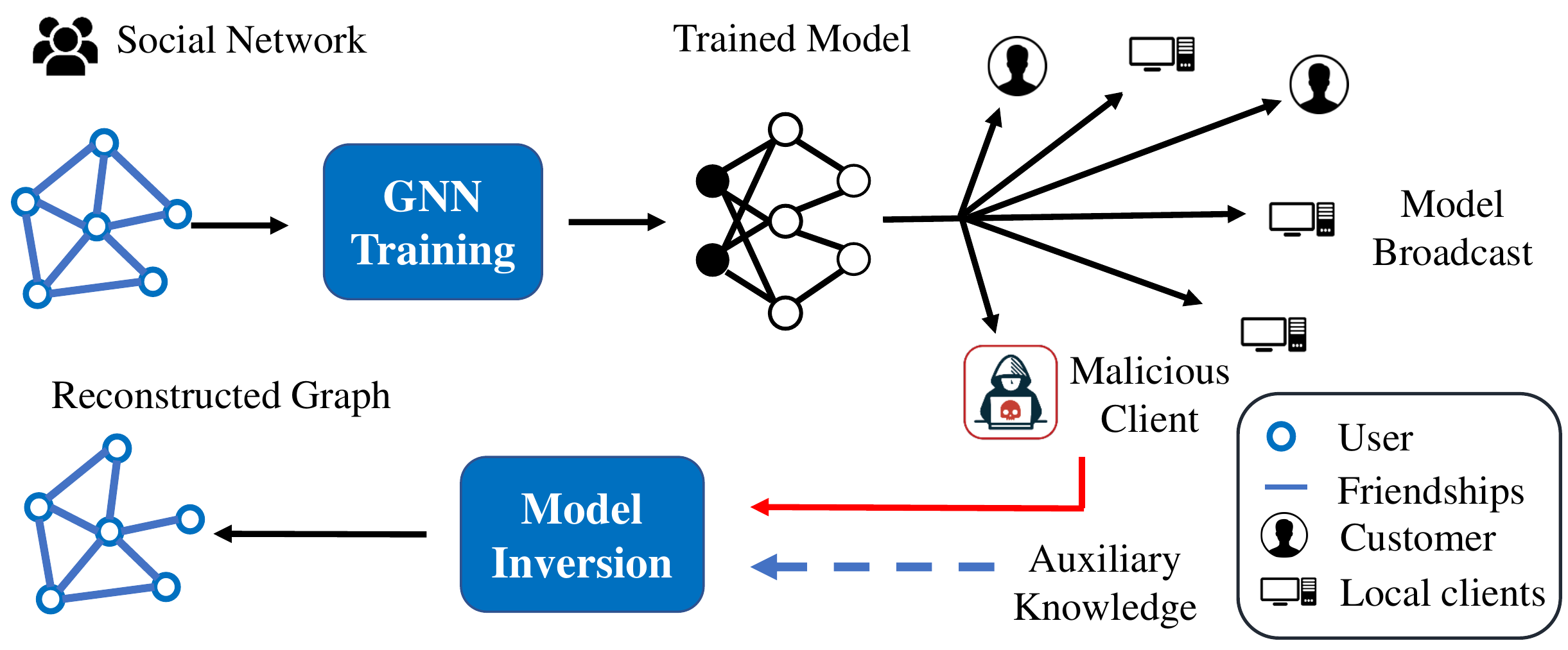}
	\caption{One motivation scenario in social networks where the attacker aims to reconstruct the friendship through the model inversion attack. Users'  friendships are sensitive relational data, and users want to keep them private. Sometimes social network data is collected with user permission to train GNN models for better service. If the attacker can obtain the trained GNN model from malicious clients, with some auxiliary knowledge crawled from the internet, model inversion attack can be performed to reconstruct the friendships among users.
	}
	\vspace{-0.5cm}
	\label{motivation}
\end{figure}
However, the fact that many GNN-based applications rely on processing sensitive graph data raises great privacy concerns. Attackers may exploit the output (i.e., black-box attack) or the parameters (i.e., white-box attack) of trained graph neural networks to potentially reveal sensitive information in the training graph data. Of particular interest to this paper is the model inversion attack which aims to extract sensitive features of training data given output labels and partial knowledge of non-sensitive features. Studying model inversion attack on GNNs helps us understand the vulnerability of GNN models and enable us to evaluate and avoid privacy risks in advance.

\textbf{Motivation scenario:}
Figure \ref{motivation} shows one concrete motivation scenario of model inversion attack on GNN models. Users'  friendships are sensitive relational data, and users want to keep them private. Sometimes social network data is collected with user permission to train GNN models for better service. For example, these trained GNN models are used to classify friends or recommend advertisements. Then, the trained models are broadcast to customers or local clients. If the attacker can obtain the trained GNN model from malicious clients, with some auxiliary knowledge crawled from the internet, model inversion attack can be performed to reconstruct friendships among users.

Model inversion attack was firstly introduced by \cite{fredrikson2014privacy}, where an attacker, given a linear regression model for personalized medicine and some demographic information about a patient, could predict the patient’s genetic markers. Generally, model inversion relies on the correlation between features and output labels and tries to maximize a posteriori (MAP) or likelihood estimation (MLE) to recover sensitive features.
Recently, efforts have been made to extend model inversion to attack other machine learning models, in particular Convolutional Neural Networks (CNN) ~\cite{fredrikson2015model,aivodji2019gamin}. Thus far, most model inversion attacks are investigated in the grid-like domain (e.g., images), leaving its effect on the non-grid domain (e.g., graph-structured data) an open problem. 


We draw attention to the model inversion attack that aim to extract private graph data from GNN models. We focus on the following adversarial scenario: Given the trained GNN model and some auxiliary knowledge (node labels and attributes), the adversary wants to reconstruct all the edges among nodes in the training dataset. However, directly adapting model inversion attack to graphs leads to poor attack performance due to the following reasons. Firstly, existing model inversion attack methods can barely be applied to the graph setting due to the discrete nature of graphs. Different from the continuous image data, gradient computation and optimization on binary edges of the graph are difficult.
Secondly, current model inversion methods fail to exploit the intrinsic properties of graph such as sparsity and feature smoothness. In addition, existing model inversion attack methods cannot fully leverage the information of node attributes and GNN models. For example, node pairs with similar attributes or embeddings are more likely to have edges. 

To address the aforementioned challenges, in our preliminary work \cite{ijcai2021-516}, we proposed \textbf{Graph} \textbf{M}odel \textbf{I}nversion attack (\textbf{GraphMI}) for edge reconstruction.
GraphMI is designed with three important modules: the projected gradient module, the graph auto-encoder module, and the random sampling module.
The projected gradient module is able to tackle the edge discreteness via convex relaxation while preserving graph sparsity and feature smoothness.
The graph auto-encoder module is designed to take all the information of node attributes, graph topology, and target model parameters into consideration for graph reconstruction.
The random sampling module is used to recover the binary adjacency matrix. With these delicately designed modules, GraphMI can effectively infer the private edges in the training graph.
We systematically evaluate our attack and compare it with baseline attacks on real-world datasets from four different fields and three representative GNN methods \footnote{https://github.com/zaixizhang/GraphMI}. 
Based on GraphMI, we also investigate the relation between edge influence and model inversion risk and find that edges with greater influence are more likely to be reconstructed.

However, GraphMI is only performed in the white-box setting, where the attacker requires full access to the target model to take derivations with respect to the adjacency matrix \cite{ijcai2021-516}. Thus, it is infeasible in the more challenging hard-label black-box setting where the attacker can only query the GNN API and receive the classification results. Therefore, in this paper we extend GraphMI and propose two black-box attacks based on gradient estimation and reinforcement learning. We also explore the countermeasures against our attack including differential private training \cite{abadi2016deep} and graph preprocessing based on edge perturbation. Specifically in graph preprocessing, three strategies including randomly rewiring, adding new edges and randomly flipping are tried. The key idea to is to ``hide" real edges from model inversion by perturbing the graph structure. Experiments show that these defenses are not effective enough to achieve a satisfying trade-off between node classification accuracy and robustness to GraphMI.  

In summary, we  make the following contributions.
\begin{itemize}
    \item  We  propose GraphMI, a novel model inversion attack against Graph Neural Networks. Our attack is specifically designed and optimized for extracting private  graph-structured data from GNNs.
    \item We extend GraphMI to the black-box setting and propose two attack methods, gradient estimation and RL-GraphMI.
    \item We explore the countermeasures against our attack. We empirically show that our attack is still effective even with differential private training or graph preprocessing.
\end{itemize}
The rest of our paper is organized as follows: In Section 2, we review different attack and defense methods for GNNs; in Section 3, we introduce the problem formulation; in Section 4, we first introduce the design of GraphMI framework and then provide the analysis on correlation  between  Edge influence and inversion risk; in Section 5, we show how to extend GraphMI to the black-box setting based on gradient estimation and reinforcement learning; in Section 6, we experimentally demonstrate the attack performance of our method; in Section 7, we explore the countermeasures against GraphMI. Finally, we show the conclusion and discuss the future works in Section 8.

\section{Related Work}
In this section, we review research works related to our proposed attacks. We refer the readers to \cite{zhang2020deep} for an in-depth overview of different GNN models, and \cite{sun2018adversarial, jin2020adversarial} for comprehensive surveys of existing adversarial attacks and defense strategies on GNNs.
\subsection{Adversarial Attacks on GNNs}
Existing studies have shown that GNNs are vulnerable to adversarial attacks \cite{dai2018adversarial, chang2020restricted, ma2019attacking, wang2019attacking, li2021adversarial, feng2019graph, xu2019topology, zugner2018adversarial,ma2020black, mu2021hard, xu2021robustness}, which deceive GNN models to make wrong predictions for graph classification or node classification. Depending on the stages when these attacks occur, these adversarial attacks can be classified into training-time poisoning attacks \cite{liu2019unified, wang2019attacking, zugner2019adversarial} and testing time adversarial attacks \cite{chen2020link, ma2020black, chen2018fast}. Based on the attacker's knowledge, adversarial attacks can also be categorized into white-box attacks \cite{xu2019topology, zugner2018adversarial} and black-box attacks \cite{ma2020black, mu2021hard, chang2020restricted}.
\subsection{Privacy Attacks on GNNs}
Based on the attacker's goal, privacy attacks can be categorized into several types \cite{rigaki2020survey}, such as membership inference attack \cite{shokri2017membership}, model extraction attack \cite{tramer2016stealing}, and model inversion attack \cite{fredrikson2015model}. Membership inference attack tries to determine whether one sample was used to train the machine learning model; Model extraction attack is one black-box privacy attack. It tries to extract information of model parameters and reconstruct one substitute model that behaves similarly to the target model. Model inversion attack, which is the focus of this paper, aims to reconstruct sensitive features corresponding to labels of target machine learning models.\par
Model inversion attack was firstly presented in \cite{fredrikson2014privacy} for linear regression models.  M. Fredrikson \textit{et al.} \cite{fredrikson2015model} extended model inversion attack to extract faces from shallow neural networks. They cast the model inversion as an optimization problem and solve the problem by gradient descent with modifications to the images. Furthermore, several model inversion attacks in the black-box setting or assisted with generative methods are proposed~\cite{aivodji2019gamin,zhang2020secret} in the image domain. In addition, some recent papers started to  study the factors affecting a model’s vulnerability to model inversion attacks theoretically \cite{wu2016methodology,yeom2018privacy}.\par
Different from adversarial attacks, only a few studies \cite{wu2020model, he2021stealing, duddu2020quantifying, zhang2021inference, hsieh2021netfense} focused on privacy attacks on GNNs. For instance, Wu \emph{et al.} \cite{wu2020model} discussed GNN model extraction attack. Given  various levels of background knowledge, they tried to gather both the input-output query pairs and the graph structure to reconstruct a duplicated GNN model. He \emph{et al.} \cite{he2021stealing} and Duddu \emph{et al.} \cite{duddu2020quantifying} tried to infer the edges based on node embeddings from GNNs. Zhang \emph{et al.} \cite{zhang2021inference} proposed three inference attacks against graph embeddings.
Thus far, no researcher has explored model inversion attacks against Graph Neural Networks.
\subsection{Defense of Attacks on GNNs}
Most defense mechanisms focus on mitigating adversarial attacks by \emph{e.g.} graph sanitization \cite{wu2019adversarial}, adversarial training \cite{dai2019adversarial,multiview}, and certification of robustness \cite{bojchevski2019certifiable}.
On the other hand, to defend against privacy attacks, one of the most popular proposed countermeasures is differential privacy (DP) \cite{dwork2014algorithmic}. 
Differential privacy is a notion of privacy that entails that the outputs
of the model on neighboring inputs are close. This privacy requirement ends up obscuring the impact of any individual training instance on the model output.
Abadi \emph{et al.} \cite{abadi2016deep} proposed DP-SGD, which ensures $(\epsilon, \delta)-$DP through adding Gaussian noise to clipped gradients in training iterations. Hay \emph{et al.} \cite{hay2009accurate} firstly extend differential privacy to the graph setting. Since then, extensive efforts have been made on computing graph statistics under edge or node differential privacy such as
degree distribution \cite{hay2009accurate}, cut queries \cite{blocki2012johnson}, and sub-graph counting queries \cite{blocki2013differentially}. These methods are useful for analyzing graph statistics but insufficient for training a GNN model with DP. 
To evaluate the effectiveness of GraphMI, we adapt DP-SGD and propose three graph preprocessing methods as countermeasures.

\section{Problem Formulation}
We begin by providing preliminaries on GNNs. Then we will present the problem definition of model inversion attack on graph neural networks.
\subsection{Preliminaries on GNNs}
Before defining GNN, we firstly introduce the following notations of graph. Let $G=(\mathcal{V,E})$ be an undirected and unweighted graph, where $\mathcal{V}$ is the vertex (i.e. node) set with size $|\mathcal{V}| = N$, and $\mathcal{E}$ is the edge set. Denote $A\in\{0,1\}^{|\mathcal{V}|\times |\mathcal{V}|}$ as an adjacent matrix containing information of network topology and $X \in \mathbb{R}^{ |\mathcal{V}|\times l}$ as a feature matrix with dimension $l$. A graph can also be expressed as $G =\{A, X\}$. 
One task that GNN models are commonly used for is semi-supervised node classification~\cite{kipf2016semi}.
In a GNN model for node classification, each node $i$ is associated with a feature vector $\textbf{x}_i \in \mathbb{R}^l$ and a scalar label $y_i$. Given the adjacency matrix $A$ and the labeled node data $\{(\textbf{x}_i,y_i)\}_{i \in \mathcal{T}}$ ($\mathcal{T}$ denotes the training set), GNN is trained to predict the classes of unlabeled nodes. 
\par
Formally, the $k$-th layer of a GNN model obeys the message passing rule and can be modeled by one message passing phase and one aggregation update phase:
\begin{equation}
    m_v^{k+1}= \mathop{\sum}_{u\in \mathcal{N}(v)}\mathbf{M}_k(h_v^k, h_u^k, e_{uv}),
\end{equation}
\begin{equation}
    h_v^{k+1}= \mathbf{U}_k(h_v^k, m_v^{k+1}),
\end{equation}
where $\mathbf{M}_k$ denotes the message passing function and $\mathbf{U}_k$ is the vertex update function. $\mathcal{N}(v)$ is the neighbors of $v$ in graph $G$. $h_v^k $ is the feature vector of node $v$ at layer $k$ and $ e_{uv}$ denotes the edge feature. $h^0_v = \textbf{x}_v $ is the input feature vector of node $v$. \par
Specifically, Graph Convolutional Network (GCN)~\cite{kipf2016semi}, a well-established method for semi-supervised node classification, obeys the following rule to aggregate neighboring features:
\begin{equation}
    H^{k+1}= \sigma \big (\hat D^{-\frac{1}{2}} \hat A \hat D^{-\frac{1}{2}} H^k W^k \big),
\end{equation}
where $\hat A = A + I_N$ is the adjacency matrix of the graph $G$ with self connections added and $\hat D$ is a diagonal matrix with $\hat D_{ii} = \sum_j \hat A_{ij}$. $\sigma(\cdot)$ is the ReLU function. $H^k$ and $W^k$ are the feature matrix and the trainable weight matrix of the $k$-th layer respectively. $H^0 = X$ is the input feature matrix. Note that in most of this paper, we focus on two-layer GCN for the node classification. Later, we show that our graph model inversion attack can extend to other types of GNNs, including GAT~\cite{velivckovic2017graph} and GraphSAGE~\cite{hamilton2017inductive}.

\subsection{Problem Definition}
\begin{figure}[!t]
	\centering
	\includegraphics[width=0.48\textwidth]{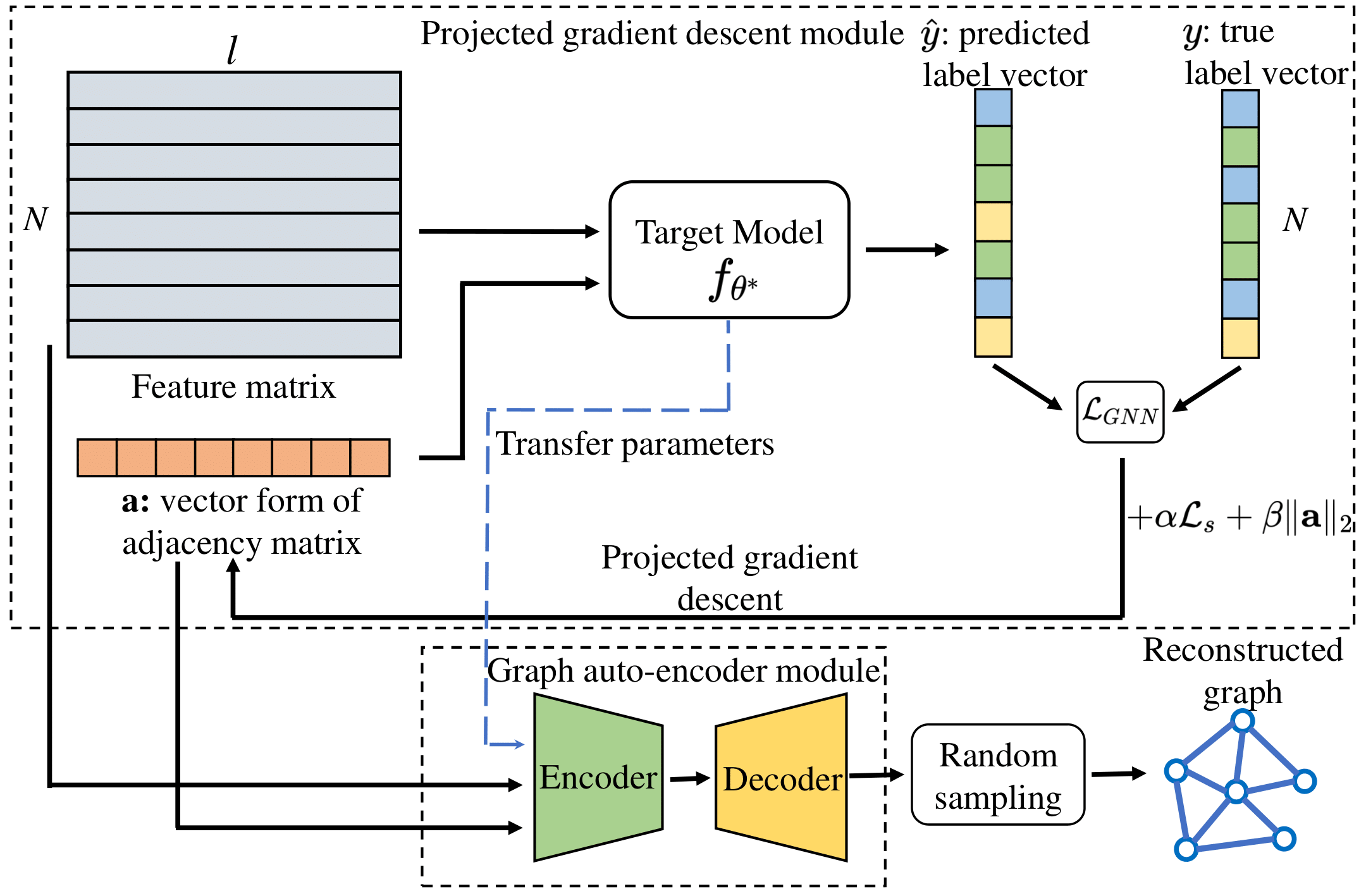}
	\caption{Illustration of the proposed GraphMI. Generally, GraphMI consists of three modules: projected gradient descent module, graph auto-encoder module and random sampling module. Firstly, the projected gradient descent module gets the initial reconstructed adjacency matrix through projected gradient descent. Then, the graph auto-encoder module post-process the reconstructed adjacency matrix. Finally, we can sample a discrete adjacency matrix from the edge probability matrix.}
	\label{illustration}
\end{figure}
We refer to the trained model subjected to model inversion attack as the target model. In this paper, we will firstly train a GNN for node classification task from scratch as the target model. We assume a threat model similar to the existing model inversion attacks \cite{fredrikson2015model}.

\subsubsection{Attacker's Knowledge and Capability}
In this paper, we investigate model inversion attacks against GNN in both white-box and hard-label black-box settings. In the white-box setting, the attacker is assumed to have full access to the target model $f_\theta$ including the model architecture and all the parameters $\theta$. In addition to the target model $f_\theta$, the attacker may have other auxiliary knowledge to facilitate model inversion such as node labels, node attributes, node IDs or edge density. We will discuss the impact of auxiliary knowledge and the number of node labels on attack performance in the following sections. In the hard-label black-box setting, the attacker is only allowed to query the target GNN model $f_\theta$ with an input graph and obtain the predicted labels.
\subsubsection{Objective of Attack}
Assuming the target GNN model $f_\theta$ is first trained on $G=\{A,X\}$, the objective of attacker $\mathcal{A}$ is:
\begin{align}
    \textup{max } s(\widetilde{A}, A), s.t.~\widetilde{A}=\mathcal{A}(X,Y,f_\theta),
\end{align}
where $Y$ is the vector of node labels and $s(\cdot,\cdot)$ is the similarity function to measure the similarity of adjacency matrices e.g. Weisfeiler-Lehman kernel \cite{shervashidze2011weisfeiler}. $\widetilde{A}$ is the reconstructed adjacency matrix by the attacker $\mathcal{A}$.

\begin{algorithm}[tb]
\caption{GraphMI}
\label{code}
\textbf{Input}: Target GNN model $f_{\theta}$;
Node label vector $Y$;
Node feature matrix $X$;
Learning rate, $\eta_t$;
Iterations $T$;\\
\textbf{Output}: Reconstructed $A$
\begin{algorithmic}[1] 
\STATE $\mathbf{a}^{(0)}$ is set to zeros
\STATE Let $t=0$
\WHILE{t \textless T}
\STATE Gradient descent: $\mathbf{a}^{(t)}= \mathbf{a}^{(t-1)}- \eta_t \nabla \mathcal{L}_{attack}(\mathbf{a})$;
\STATE Call Projection operation in (\ref{projection})
\ENDWHILE
\STATE Call Graph auto-encoder module in (\ref{GAE})
\STATE Call Random sampling module in Algorithm 2.
\STATE \textbf{return} $A$
\end{algorithmic}
\end{algorithm}

\section{White-Box Attack: GraphMI}
Next we introduce GraphMI, our proposed white-box model inversion attack against GNN models. The target model $f_\theta$ is firstly trained to minimize the loss function $\mathcal{L}(X, Y, A, \theta)$:
\begin{equation}
    \theta^*={\textup{arg }}\mathop{\textup{min}}\limits_{\theta}\mathcal{L}(X, Y, A, \theta).
\end{equation}
Without ambiguity, we still use $\theta$ as the trained model parameters in the rest of this paper. Then, given the trained model and its parameters, GraphMI aims to find the adjacency matrix $\widetilde{A}$ that maximizes the posterior possibility:
\begin{equation}
    \widetilde{A}={\textup{arg }}\mathop{\textup{max}}\limits_{A}P(A | X, Y, f_\theta).
\end{equation}
\subsection{Attack Overview}
Figure \ref{illustration} shows the overview of GraphMI. Generally, GraphMI is one optimization-based attack method, which firstly employs projected gradient descent on the graph to find the ``optimal'' network topology for node labels. Then the adjacency matrix and feature matrix will be sent to the graph auto-encoder module of which parameters are transferred from the target model. Finally, we can interpret the optimized graph as the edge probability matrix and sample a binary adjacency matrix. We summarize GraphMI in Algorithm \ref{code}.
\subsection{Details of Modules}
\subsubsection{Projected Gradient Descent Module} We treat model inversion on GNNs as one optimization problem: given node features or node IDs, we want to minimize the cross-entropy loss between true labels $y_i$ and predicted labels $\hat{y_i}$ from the target GNN model $f_{\theta}$. The intuition is that the reconstructed adjacency matrix will be similar to the original adjacency matrix if the cross-entropy loss between true labels and predicted labels is minimized. The attack loss on node $i$ is denoted by $\ell_i(A, f_{\theta}, \textbf{x}_i, y_i)$ where $A$ denotes the reconstructed adjacency matrix for brevity, $\theta$ is the model parameter of the target model $f$ and $\textbf{x}_i$ denotes the node feature vector of node $i$. The attack objective function can be formulated as:
\begin{equation}
    \begin{split}
    \mathop{\textup{min}}\limits_{A \in \{0,1\}^{N\times N}} \mathcal{L}_{GNN}(A) &=\frac{1}{N}\sum_{i=1} ^N \ell_i(A, f_{\theta}, \textbf{x}_i, {y_i}) \\s.t.\quad A &= A^\top.
\end{split}
\label{equ1}
\end{equation}
\par
In many real-world graphs, such as social networks, citation networks, and web pages, connected nodes are likely to have similar features \cite{wu2019adversarial}. Based on this observation, we need to ensure the feature smoothness in the optimized graph. The feature smoothness can be captured by the following loss term $\mathcal{L}_s$:
\begin{equation}
    \mathcal{L}_{s} = \frac{1}{2}\sum_{i, j =1} ^N A_{i, j}(\textbf{x}_i-\textbf{x}_j)^2,
\end{equation}
where $A_{i,j}$ indicates the connection between node $v_i$ and $v_j$ in the optimized graph and $(\textbf{x}_i-\textbf{x}_j)^2$ measures the feature difference between $v_i$ and $v_j$. $\mathcal{L}_s$ can also be represented as:
\begin{equation}
    \mathcal{L}_s = tr(X^\top L X),
\end{equation}
where $L=D-A$ is the laplacian matrix of $A$ and $D$ is the diagonal matrix of $A$. In this paper, to make feature smoothness
independent of node degrees, we use the normalized lapacian matrix $\hat{L} = D^{-1/2}LD^{-1/2}$ instead: 
\begin{equation}
    \mathcal{L}_s = tr(X^\top \hat{L} X)=\frac{1}{2}\sum_{i, j =1} ^N A_{i, j}(\frac{\textbf{x}_i}{\sqrt{d_i}}-\frac{\textbf{x}_j}{\sqrt{d_j}})^2,
\end{equation}
where $d_i$ and $d_j$ denote the degree of node $v_i$ and $v_j$. The authors also notice that recent studies \cite{jin2021universal, lim2021large} show that some heterophily graphs do not have the characteristic of feature smoothness. We left dealing with this class of graph for the future work.
To encourage the sparsity of graph structure, $F$ norm of adjacency matrix $A$ is also added to the loss function as one regularization term. The final objective function is:
\begin{equation}
\begin{split}
    {\textup{arg }}\mathop{\textup{min}}\limits_{A \in \{0,1\}^{N\times N}}\mathcal{L}_{attack} &=\mathcal{L}_{GNN}+\alpha \mathcal{L}_s+\beta \|A\|_F
    \\s.t. \quad A &= A^\top,\label{equ3}
\end{split}
\end{equation}
where $\alpha$ and $\beta$ are hyper-parameters that control the weight of feature smoothing and graph sparsity.
Solving equation (\ref{equ3}) is a combinatorial optimization problem due to edge discreteness. For ease of gradient computation and update, we firstly replace the symmetric reconstructed adjacency matrix $A$ with its vector form $\mathbf{a}$ that consists of $n\coloneqq N(N-1)/2$ unique variables in $A$. Adjacency matrix $A$ and vector $\mathbf{a}$ can be converted to each other easily, which ensures the optimized adjacency matrix is symmetric. Then we relax $\mathbf{a} \in \{0,1\}^n$ into convex space $\mathbf{a} \in [0,1]^n$. We can perform model inversion attack by firstly solving the following optimization problem:
\begin{equation}
    {\textup{arg }}\mathop{\textup{min}}\limits_{\mathbf{a} \in [0,1]^n}\mathcal{L}_{attack} =\mathcal{L}_{GNN}+\alpha \mathcal{L}_s+\beta \|\mathbf{a}\|_2.
\label{equ4}
\end{equation}

\par The continuous optimization problem (\ref{equ4}) is solved by projected gradient descent (PGD):
\begin{equation}
    \mathbf{a}^{t+1} = P_{[0,1]} \big[\mathbf{a}^t -\eta_t g_t \big],
    \label{pgd}
\end{equation}
where $t$ is the iteration index of PGD, $\eta_t$ is the learning rate, $g_t$ is the gradients of loss $\mathcal{L}_{attack}$ in \ref{equ3} evaluated at $\textbf{a}^t$, and
\begin{equation}
P_{[0,1]}[x]=\left\{
\begin{array}{rcl}
0 & & x < 0\\
1 & & x > 1\\
x & & otherwise.
\end{array} \right.
\label{projection}
\end{equation}
is the projection operator.
\subsubsection{Graph Auto-encoder Module} In GraphMI, we propose to use graph auto-encoder (GAE) \cite{kipf2016variational} to post-process the optimized adjacency matrix $A$. GAE is composed of two components: encoder and decoder. We transfer part of the parameters from the target model $f_{\theta}$ to the encoder. Specifically, feature matrix and adjacency vector $\mathbf{a}$ are sent to the $f_{\theta}$ and the node embedding matrix $Z$ is generated by taking the penultimate layer of the target model $f_{\theta}$, which is denoted as $H_{\theta}(\mathbf{a}, X)$. Then the decoder will reconstruct adjacency matrix $A$ by applying logistic sigmoid function to the inner product of $Z$:
\begin{equation}
    A={\rm sigmoid}(ZZ^\top),  {\rm with}~ Z=H_{\theta}(\mathbf{a}, X).\label{GAE}
\end{equation}
\par The node embeddings generated by the graph auto-encoder module encode the information from node attributes, graph topology, and the target GNN model. Intuitively, node pairs with close embeddings are more likely to form edges. In experiments, we conduct ablation studies to show the impact of the graph auto-encoder module on the attack performance.
\subsubsection{Random Sampling Module} After the optimization problem is solved, the solution $A$ can be interpreted as a probabilistic matrix, which represents the possibility of each edge. We could use random sampling to recover the binary adjacency matrix. Firstly, we need to calculate the number of sampling edges $\lfloor \rho n \rfloor$ given the estimated graph density $\rho$. The estimated graph density $\rho$ can be obtained from public graphs from the same domain. In the experiments, we will show how the discrepancy between the estimated density and the ground truth density affects the attack performance. Generally, GraphMI is robust to the graph density estimation. Then we take $K$ trials to sample edges based on the probabilistic matrix. Finally, we choose the adjacency matrix with the smallest attack loss $\mathcal{L}_{attack}$. More details are shown in Algorithm \ref{random sampling}. 

\begin{algorithm}[t]
\caption{Random sampling from probabilistic vector to binary adjacency matrix}
\label{random sampling}
\textbf{Input}: Probabilistic vector $\mathbf{a}$, number of trials $K$\\
\textbf{Parameter}: Edge density $\rho$\\
\textbf{Output}: Binary matrix $A$;
\begin{algorithmic}[1] 
\STATE Normalize probabilistic vector: $\hat{\mathbf{a}}$ = $\mathbf{a}/\|\mathbf{a}\|_1$
\FOR{k = 1,2 $\cdots$ K}
\item
Draw binary vector $\mathbf{a}^{(k)}$ by sampling $\lfloor \rho n \rfloor$ edges according to probabilistic vector $\hat{\mathbf{a}}$.
\ENDFOR
\STATE Choose a vector $\mathbf{a}^*$ from $\{\mathbf{a}^{(k)}\}$ which yields the smallest loss $\mathcal{L}_{attack}$. 
\STATE Convert $\mathbf{a}^*$ to binary adjacency matrix $A$ 
\STATE \textbf{return} $A$
\end{algorithmic}
\end{algorithm}

\subsection{Analysis on Correlation between Edge Influence and Inversion Risk}
In previous work \cite{wu2016methodology}, researchers found feature influence to be an essential factor in incurring privacy risk.  In our context of graph model inversion attack, sensitive features are edges. Here we want to characterize the correlation between edge influence and inversion risk. Given label vector $Y$, adjacency matrix $A$ and feature matrix $X$, the performance of target model $f_{\theta}$ for the prediction can be measured by prediction accuracy:
\begin{equation}
    ACC(f_{\theta}, A, X)=\frac{1}{N}\sum_{i=1}^N\mathbf{1}[f_{\theta}^i (A, X)=y_i],
\end{equation}
where, $f_{\theta}^i (A, X)$ is the predicted label for node $i$. The influence of edge $e$ can be defined as:
\begin{equation}
     {\mathcal I}(e)=ACC(f_{\theta}, A, X) - ACC(f_{\theta}, A_{-e}, X),
     \label{edge}
\end{equation} 
where $A_{-e}$ denotes removing the edge $e$ from the adjacency matrix $A$. \cite{wu2016methodology} proposed to use adversary advantage to characterize model inversion risk of features. The model inversion advantage $Adv$ of adversary $\mathcal{A}$ is defined to be $P[\mathcal{A}(X, f_{\theta})=e]- 1/2$, where $P[\mathcal{A}(X, f_{\theta})= e]$ is the probability that adversary $\mathcal{A}$ correctly infer the existence of edge $e$. Next, we introduce our theorem.
\begin{theorem}
$Adv(e_i) \le Adv(e_j),~ \forall e_i, e_j \in \mathcal{E} \cap \mathcal{I}(e_i) \le \mathcal{I}(e_j)$ i.e. The adversary advantage is greater for edges with greater influence.
\label{theorem1}
\end{theorem}
\begin{proof}
Our detailed proof in shown in Appendix A. 
\end{proof}
Intuitively, edges with greater influence are more likely to be recovered by GraphMI because these edges have a greater correlation with the model output. In the following section, we will validate our theorem with experiments.

\section{Black-box Attack}
We illustrate the white-box attack GraphMI in the above section. In this section, we consider the strictest hard-label black-box setting. Specifically, the attacker is only allowed to query the target GNN model $f_{\theta}$ with an input graph and obtain only the predicted hard label (instead of a confidence vector that indicates the probabilities that the graph belongs to each class) for
the input graph. All the other knowledge, e.g., training graphs, structures, and parameters of the target GNN model are unavailable to the attacker. Such a black-box situation brings more challenges to the design of model inversion attacks because the attacker is unable to take the derivative of the adjacency matrix directly. In this section, we propose two methods based on gradient estimation and reinforcement learning.
\subsection{Gradient Estimation}
We first introduce the projected gradient estimation algorithm to solve the problem of black-box attack. With zeroth order oracle, we can estimate the gradient
of $\mathcal{L}_{attack}$ in Equ.\ref{equ4} via computing $(\mathcal{L}_{attack}(a+\mu u)-\mathcal{L}_{attack}(a-\mu u))/2\mu$, where $u$ denotes a normalized direction vector sampled randomly from a uniform distribution over a unit sphere \cite{liu2018zeroth, shamir2017optimal, gao2018information}, and $\mu$ is a step constant. Note that since the attacker is only accessible to the predicted hard labels, $\mathcal{L}_{GNN}$ in $\mathcal{L}_{attack}$ cannot be computed in the black-box setting. We use the error rate $\frac{1}{N}\sum_{i=1}^N\mathbf{1}(f_{\theta}^i (A, X) \neq y_i)$ instead. To ease the noise of gradients, we average the gradients in different directions to estimate the derivative of $\mathcal{L}_{attack}$ as follows:
\begin{equation}
    \hat{\nabla} \mathcal{L}_{attack}(a) = \frac{d}{q} \sum_{j=1}^q \frac{\mathcal{L}_{attack}(a+\mu u_j)-\mathcal{L}_{attack}(a-\mu u_j)}{2\mu} u_j,
\end{equation}
where $q$ is the number of queries to estimate the gradient $\hat{\nabla} \mathcal{L}_{attack}$ and $d$ is the dimension of optimization variables. After computing the estimated gradient, the attacker update the vector of adjacency matrix by  projected gradient descent in Equ.\ref{pgd}. The following
theorem shows the convergence guarantees of our zeroth-order projected gradient descent for the black-box attack.
\begin{assumption}
The loss function $\mathcal{L}_{attack}$ is $L_1$-Lipschitz continuous for a finite positive constant $L_1$.
\end{assumption}
\begin{assumption}
The loss function $\mathcal{L}_{attack}$ is differentiable and has $L_2$-Lipschitz continuous gradients for a finite positive constant $L_2$.
\end{assumption}

\begin{theorem}
\label{convergence therom}
Suppose Assumption 1$\&$2 hold. If we randomly pick $a_r$, whose dimension is $d$, from $\{a_t\}_{t=0}^{T-1}$ with probability $P(r=t) = \frac{2\eta_t-L_2\eta_t^2}{\sum_{t=0}^{T-1}2\eta_t-L_2\eta_t^2}$ Then, we have the following bound on the convergence rate with $\eta_t = \mathcal{O}(\frac{1}{\sqrt{T}})$ and $\mu = \mathcal{O}(\frac{1}{\sqrt{dq}})$: 
\begin{equation}
    \mathbb{E} [\|P_{[0,1]}(a_r, \hat{\nabla} \mathcal{L}_{attack}(a_r), \eta_r)\|^2] = \mathcal{O}(\frac{1}{\sqrt{T}}+\frac{d+q}{q}),
\end{equation}
where $P_{[0,1]}(a_r, \hat{\nabla} \mathcal{L}_{attack}(a_r), \eta_r) \coloneqq (1/\eta_r) [a_r - P_{[0,1]}(a_r - \eta_r \hat{\nabla} \mathcal{L}_{attack}(a_r))$ is the rectified gradient.
\end{theorem}
\begin{proof}
We defer the proof to the Appendix B.
\end{proof}
\subsection{RL-GraphMI}
Deep reinforcement learning (DRL) offers a promising approach for realizing black-box graph model inversion attacks. Firstly, the addition of edges between nodes can be naturally modeled by actions in a reinforcement learning model. Secondly, the non-linear mapping between the graph structure and its low-dimensional representation can be well captured by a deep neural network.\par
The key idea behind RL-GraphMI is to use a deep reinforcement learning agent to iteratively perform actions aimed at reconstructing the edges of the training graph. Formally, We model the reconstruction of edges by a Finite Horizon Markov Decision Process $(S, A, P, R, \gamma)$ where $S$ denotes the set of states, $A$ the set of actions, $P$ the matrix of state transition probabilities,
$R$ is the reward function, and $\gamma < 1$ is a discounting factor that discounts delayed reward relative to immediate reward.
\subsubsection{States and Actions}
The state $s_t$ contains the intermediate reconstructed graph $G_t$ and the node labels ${y_i}$. We use a two layer graph convolutional network and a two layer neural network to encode $G_t$ and ${y_i}$ respectively. The embedding of the node $e(v)$ is obtained by concatenating the structure embedding from $G_t$ and label embedding from $y_i$. The embedding of $s_t$ is obtained by averaging all the node embeddings.\par
A single action at time step $t$ is $a_t \in V \times V$.  However, simply performing actions in space $\mathcal{O}(|V|^2)$ is too expensive. Following \cite{dai2018adversarial} and \cite{sun2020non}, we adopt a hierarchical decomposition of actions to reduce the action space. At time $t$, the DRL agent first
performs an action $a_t^{(1)}$ to select the first node and then pick the second node using $a_t^{(2)}$. After $a_t^{(1)}$ and $a_t^{(2)}$, a new edge is added between the selected two nodes if there was no edge. With the hierarchy action $a_t= (a_t^{(1)}, a_t^{(2)})$, the trajectory of the proposed MDP is $(s_0, a_0^{(1)}, a_0^{(2)}, r_0, s_1, \cdots, s_{T-1}, a_{T-1}^{(1)}, a_{T-1}^{(2)}, s_T, r_T)$.
\subsubsection{Policy Network}
After \cite{dai2018adversarial} and \cite{sun2020non}, we use Q-learning to find a policy that is optimal in the sense that it maximizes the expected value of the total reward over any and all successive steps, starting from the current state. Q-learning is an off-policy optimization which aims to satisfy the the Bellman optimality equation as follows:
\begin{equation}
    Q^*(s_t, a_t) = r(s_t, a_t)+\gamma \mathop{\textup{max}}\limits_{a_t'} Q^*(s_{t+1}, a'_t).
\end{equation}
The greedy policy to select the action $a_t$ with respect to $Q^*$:
\begin{equation}
    a_t = \pi (s_t) = \mathop{\textup{arg max}}\limits_{a_t} Q^*(s_t, a_t).
\end{equation}
In experiments, We adopt a hierarchical Q function $Q=\{Q^{(1)}, Q^{(2)}\}$ for modeling the two actions:
\begin{equation}
    Q^{(1)}(s_t, a_t^{(1)}; \phi^{(1)}) = W_1^{(1)}\sigma(W_2^{(1)}[e(v_{a_t^{(1)}})\| e(s_t)]),
    \label{Q1}
\end{equation}
where $\phi^{(1)} = \{W_1^{(1)}, W_2^{(1)}\}$ denotes the trainable weights and $\|$ is the concatenation operation. $e(s_t)$ is the state embedding and $e(v_{a_t^{(1)}})$ is the node embedding of first selected node. 
With the first action $a_t^{(1)}$ selected, the DRL agent picks the second action $a_t^{(2)}$ based on $Q^{(2)}$ as follows:
\begin{equation}
    Q^{(2)}(s_t, a_t^{(1)}, a_t^{(2)}; \phi^{(2)}) = W_1^{(2)}\sigma(W_2^{(2)}[e(v_{a_t^{(2)}})\|e(v_{a_t^{(1)}})\| e(s_t)]),
    \label{Q2}
\end{equation}
where $\phi^{(2)} = \{W_1^{(2)}, W_2^{(2)}\}$ denotes the trainable weights. The action value function $Q^{(2)}$ scores the candidate nodes for establishing an
edge based on the state $s_t$, and $a_t^{(1)}$.
\subsubsection{Reward Function}
We design the reward function $r(s_t, a_t^{(1)}, a_t^{(2)})$ based on accuracy $ACC_t = \frac{1}{N}\sum_{i=1}^N\mathbf{1}[f_{\theta}^i (A_t, X)=y_i)]$ where $A_t$ is the adjacency matrix of $G_t$. However, direct use of $ACC_t$ as the reward could slow down training since the accuracy might not differ significantly between two consecutive states. Hence, we design a guiding binary reward $r_t$ to be 1 if the action $a_t = (a_t^{(1)}, a_t^{(2)})$ can increase the accuracy at time t, and to be -1 otherwise:
\begin{equation}
    r(s_t, a_t^{(1)}, a_t^{(2)}) = \left\{
\begin{array}{rcl}
1 & & ACC_t \ge ACC_{t-1}\\
-1 & & otherwise.
\end{array} \right.
\label{reward}
\end{equation}
\subsubsection{Terminal}
In RL-GraphMI, the DRL agent stops to take further actions once the number of edges reaches the predefined number which can be estimated by the graph density. 
\begin{algorithm}[t]
\caption{The training algorithm of RL-GraphMI}
\label{algo:RL}

\textbf{Input}:Target GNN model $f_{\theta}$;
Node label vector $Y$;
Node feature matrix $X$;
Learning rate, $\eta_t$;
Training Iterations $T$;
Maximum Edges $\Delta$;
Update Period $C$\\ 
\textbf{Output}:Reconstructed A
\begin{algorithmic}[1]
\STATE Initialize action-value function Q with random parameters $\phi$
\STATE Set target function $\hat{Q}$ with parameters $\phi^{-}=\phi$
\STATE Initialize replay memory buffer $\mathcal{M}$

\WHILE{episode \textless T}
\STATE Initialize edge set of the adjacency matrix $A$
\WHILE{t \textless $\Delta$}
\STATE Compute state representation
\STATE  With probability $\epsilon$ select a random action $a_t^{(1)}$, otherwise select $a_t^{(1)}$ based on  Eq.(\ref{Q1})
\STATE	With probability $\epsilon$ select a random action $a_t^{(2)}$, otherwise select $a_t^{(2)}$ based on  Eq.(\ref{Q2})
\STATE  Compute $r_t$ according to Eq.(\ref{reward})
\STATE  Set $s_{t+1}=\{s_t, a_t^{(1)}, a_t^{(2)}\}$
\STATE  $E_A \leftarrow E_A \cup (a_t^{(1)}, a_t^{(2)})$
\STATE   Store $\{s_t, a_t^{(1)}, a_t^{(2)}, r_t,s_{t+1}\}$ in memory $\mathcal{M}$
\STATE   Sample minibatch transition randomly from $\mathcal{M}$
\STATE Update parameter according to Eq.(\ref{eq:loss})
\STATE  Every C steps $\phi^{-}=\phi$;
\ENDWHILE
\ENDWHILE
\STATE \textbf{return} $A$

\end{algorithmic}
\end{algorithm}
\subsubsection{Training Algorithm}
To train the RL-GraphMI, we adopt
the experience replay technique with memory buffer $\mathcal{M}$ \cite{mnih2015human}. The trainable parameters in RL-GraphI is denoted as $\phi$. The key idea behind experience replay is to randomizes the order of
data used by Q learning, so as to remove the correlations in the observation sequence. We simulate action selection and store the resulting data in a memory buffer $\mathcal{M}$. During training, a batch of experience $(s, a, s')$ where $a = \{a^{(1)}, a^{(2)}\}$ is drawn uniformly from the stored memory buffer $\mathcal{M}$. The Q-learning loss function is similar to \cite{mnih2015human} as:
\begin{equation}\label{eq:loss}
\mathbb{E}_{(s,a,s')\sim \mathcal{M}}[(r+\gamma \max_{a'} \hat{Q}(s',a'|\phi^{-}) - Q(s,a|\phi))^2],
\end{equation}
where $\hat{Q}$ represents the target action-value function and its parameters $\phi^{-}$ are updated with $\phi$ every C steps.
To improve the stability of the algorithm, we clip the error term between $-1$ and $+1$. The agent adopts $\epsilon$-greedy policy that select a random action with probability $\epsilon$. The overall training framework is summarized in Algorithm \ref{algo:RL}.
\subsection{Computational Complexity}
In GraphMI, we take derivatives with respect to the adjacency matrix $A$. The computational complexity of GraphMI is $N^2T$ where $N$ is the number of nodes, and $T$ is the total iterations. The computational overhead of gradient estimation and RL-GraphMI are $N^2qT$ and $N^2\Delta T$ respectively, where $q$ is the number of queries and $\Delta$ is the number of maximum edges.
\begin{table}[h]
\centering
\caption{Dataset statistics}
\begin{tabular}{ccccc}
\toprule
 Datasets        & Nodes & Edges & Classes & Features \\ \midrule
Cora     & 2,708  & 5,429  & 7       & 1,433     \\ 
Citeseer & 3,327  & 4,732  & 6       & 3,703     \\ 
Polblogs & 1,490  & 19,025  & 2       & -     \\ 
USA & 1,190  & 13,599  & 4       & -     \\
Brazil & 131  & 1,038  & 4       & -     \\
AIDS & 31,385  & 64,780  & 38       & 4     \\
ENZYMES & 19,580  & 74,564  & 3       & 18     \\
\bottomrule
\end{tabular}

\label{data}
\end{table}
\section{Attack Results}
\begin{table*}[!t]
\scriptsize
\caption{Results of model inversion attack on GCN ($\%$). We repeat 5 times and report the means and standard deviations. The best result for each dataset are bolded}
\centering

\begin{tabular}{lcccccccccccc}
\toprule
                \multirow{2}{*}{\textbf{Method}}& \multicolumn{2}{c}{\textbf{Cora}}                         & \multicolumn{2}{c}{\textbf{Citeseer}}                     & \multicolumn{2}{c}{\textbf{Polblogs}}& \multicolumn{2}{c}{\textbf{USA}} &\multicolumn{2}{c}{\textbf{Brazil}} &\multicolumn{2}{c}{\textbf{AIDS}}  \\ 
                \cmidrule(r){2-3} \cmidrule(r){4-5} \cmidrule(r){6-7} \cmidrule(r){8-9}\cmidrule(r){10-11}\cmidrule(r){12-13}
                         & AUC           & AP            & AUC      & AP            & AUC            & AP      & AUC &AP&AUC&AP&AUC &AP  \\ \midrule
Attr. Sim. & 80.3$\pm$1.3          & 80.8$\pm$1.6          & \textbf{88.9$\pm$1.7}          & \textbf{89.1$\pm$1.8} &-&-&-&-&-&-          & 73.1$\pm$1.8          & 72.7$\pm$2.0   \\ 
Emb. Sim.   & 81.6$\pm$1.1          & 82.2$\pm$1.4          & 87.2$\pm$1.5          & 86.1$\pm$1.4          & 66.9$\pm$1.6          & 67.4$\pm$1.6          & 64.8$\pm$0.8                       & 65.1$\pm$0.7                       & 67.2$\pm$1.4 &70.1$\pm$1.2&74.0$\pm$1.3&72.9$\pm$1.2                            \\ 
MAP   & 74.7$\pm$1.6          & 71.8$\pm$1.5          & 69.3$\pm$0.8          & 75.5$\pm$0.9          & 68.8$\pm$1.1          & 75.1$\pm$1.3          & 59.4$\pm$1.5                       & 61.1$\pm$1.7                       & 63.8$\pm$1.2 &66.1$\pm$1.3&64.2$\pm$1.3&65.3$\pm$1.4                           \\ 
GraphMI     & \textbf{86.8$\pm$1.3} & \textbf{88.3$\pm$1.2} & 87.8$\pm$0.6 & 88.5$\pm$0.5 & \textbf{79.3$\pm$0.6} & \textbf{79.7$\pm$0.8} & \textbf{80.6$\pm$1.4}& \textbf{81.3$\pm$1.6}  & \textbf{86.6$\pm$1.0} &\textbf{88.8$\pm$0.9}&\textbf{80.2$\pm$1.0} &\textbf{80.9$\pm$1.5}                                                   \\ \bottomrule
\vspace{-0.4cm}
\end{tabular}

\label{result}
\end{table*}
In this section, we present the experimental results to show the effectiveness of GraphMI in both white-box and hard-label black-box settings. 

\subsection{Experimental Settings}
\textbf{Datasets: }Our graph model inversion attack method is evaluated on 7 public datasets from 4 categories. The detailed statistics of them are listed in Table \ref{data}.
\begin{itemize}
    \item \textit{Citation Networks}: We use Cora and Citeseer \cite{sen2008collective}. Here, nodes are documents with corresponding bag-of-words features and edges denote citations among nodes. Class labels denote the subfield of research that
    the papers belong to.
    \item \textit{Social Networks}: Polblogs \cite{adamic2005political} is the  network of political blogs whose nodes do not have features. 
    \item \textit{Air-Traffic Networks}:  The air-traffic networks are based on flight records from USA and Brazil. Each node is an airport and
    an edge indicates a commercial airline route
    between airports. Labels denote the level of activity
    in terms of people and flights passing through an
    airport \cite{ribeiro2017struc2vec}.
    \item \textit{Chemical Networks}: AIDS \cite{riesen2008iam} and Enzymes \cite{borgwardt2005protein} are chemical datasets that contain many molecure graphs, each node is an atom and each link represents chemical bonds.
\end{itemize}

\subsubsection{Target Models:}In our evaluation, we use 3 state-of-the-art GNN models: GCN \cite{kipf2016semi}, GAT \cite{velivckovic2017graph} and GraphSAGE \cite{hamilton2017inductive}. 
The parameters of the models are the same as those set in the original papers. To train a target model, 10\% randomly sampled nodes are used as the training set. All GNN models are trained for 200 epochs with an early stopping strategy based on convergence behavior and accuracy on a validation set containing 20\% randomly sampled nodes. In GraphMI attack experiments, attackers have labels of all the nodes and feature vectors. All the experiments are conducted on 1 Tesla V100 GPU.
\subsubsection{Parameter Settings:}In experiments, we set $\alpha=0.001$, $\beta=0.0001$, $\eta_t=0.1$, $K=20$ and $T=100$ as the default setting. We show how to find optimal values for hyper-parameters in the following section. For
other parameters in gradient estimation and RL-GraphMI, we set $q=100$, $\mu=0.01$, $\Delta = \lfloor n\rho_g \rfloor$, and $C=10$ as default ($\rho_g$ is the ground truth graph density). In each experiment, we repeat the attack for 5 times
and report the average attack performance with standard deviations to ease the influence of randomness.

\subsubsection{Metrics:}Since our attack is unsupervised, the attacker cannot find a threshold to make a concrete prediction through the  algorithm. To evaluate our attack, we use AUC (area under the ROC curve) and AP (average precision) as our metrics, which is consistent with previous works \cite{kipf2016variational}. In experiments, we use all the edges from the training graph and the same number of randomly sampled pairs of unconnected nodes (non-edges) to evaluate AUC and AP. The higher AUC and AP values imply better attack performance. An AUC or AP value of 1 implies maximum performance (true-positive rate of 1 with a
false-positive rate of 0) while an AUC  or AP value of 0.5 means performance equivalent to random guessing.\par
For the black-box attack, two another metrics are considered for evaluation: 1) Average Queries (AQ), i.e., the average number of queries used in the whole attack. 2) Average Time (AT), i.e., the average time used in the whole attack. An attack has better attack performance if it has a smaller AQ and AT.
\subsubsection{Baselines:}
There are three baseline methods, Attribute Similarity (abbreviated as Attr. Sim.), Embedding Similarity (abbreviated as Emb. Sim.), and MAP.
\begin{itemize}
    \item \emph{Attribute Similarity} is measured by cosine distance among node attributes, which is commonly used in previous works \cite{he2020stealing}. A higher similarity typically indicates a higher probability of edge existence between two nodes.
    \item \emph{Embedding Similarity} calculates the cosine distance among the embedding of nodes, which is similar to the Attribute Similarity. The node embeddings can be obtained by taking the node representations from the penultimate layer of the target graph neural network.
    \item The model inversion method from \cite{fredrikson2015model} \emph{MAP} is adapted to the graph setting as another baseline.
\end{itemize}
\begin{figure}[t]
	\centering
    \includegraphics[width=0.8\linewidth]{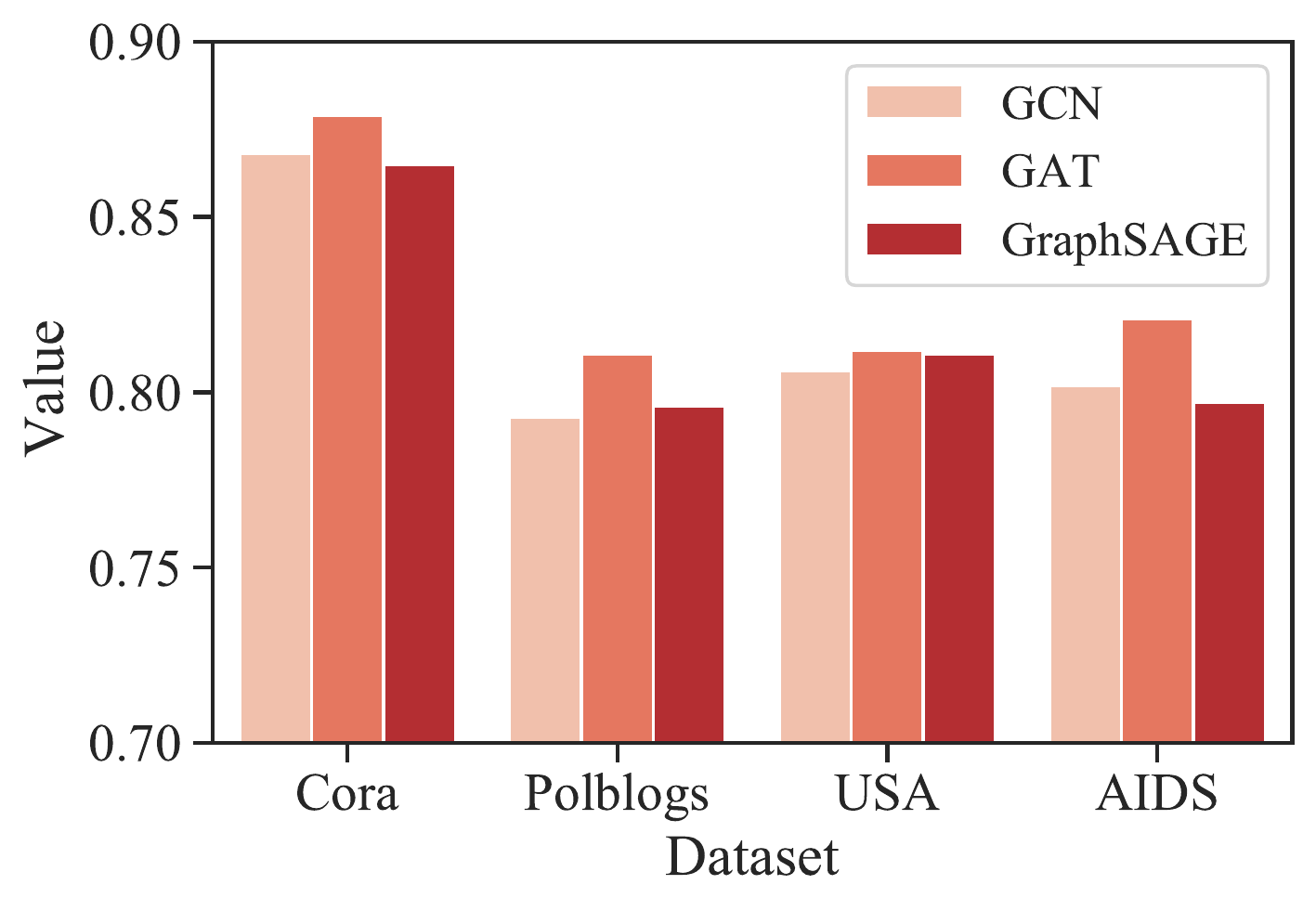}
	\caption{Attack performance of GraphMI on different Graph Neural Networks. GraphMI is agnostic to the architecture of graph neural networks and achieves success on GCN, GAT, and GraphSAGE.}
	\label{gnn type}
\end{figure}
\subsection{Results and Discussions}
\begin{table}[t]
\caption{Compare the reconstructed graph and the original graph with respect to graph isomorphism and macro-level graph statistics ($\%$). We repeat 5 times and report the means and standard deviations.}
\begin{tabular}{cccccc}
\toprule
Datasets                  & Iso. & Degree   & LCC Dist.  & BC Dist. & CC Dist. \\ \midrule
\textbf{Cora}    & 93.7$\pm$1.3      & 93.2$\pm$0.4      & 99.9$\pm$0.1      & 99.9$\pm$0.1       &99.9$\pm$0.1        \\
\textbf{Citeseer} &  94.6$\pm$1.6   & 96.2$\pm$0.8      & 99.9$\pm$0.1      & 99.9$\pm$0.1       & 99.7$\pm$0.2        \\
\textbf{Polblogs}   & 84.8$\pm$1.5   & 81.6$\pm$2.6  & 99.4$\pm$0.2 & 99.9$\pm$0.1       & 75.8$\pm$4.2        \\
\textbf{USA}                     & 86.5$\pm$0.9                         & 88.5$\pm$1.0  & 99.5$\pm$0.1 & 99.8$\pm$0.2       & 86.5$\pm$1.2        \\
\textbf{Brazil}                    & 86.9$\pm$1.1                         & 91.2$\pm$1.3    & 99.7$\pm$0.1  & 99.8$\pm$0.1       & 77.6$\pm$4.5        \\
\textbf{AIDS}            & 87.8$\pm$1.4                         & 84.7$\pm$1.9 & 99.2$\pm$0.3 & 99.6$\pm$0.1      & 80.2$\pm$3.7        \\
\textbf{Enzymes}                  & 82.3$\pm$1.5            & 75.5$\pm$3.4 & 98.8$\pm$0.2 & 98.1$\pm$0.1       & 78.7$\pm$2.9       \\ \bottomrule
\end{tabular}
\label{iso}
\end{table}
\subsubsection{Attack Performance} Results for model inversion attacks on GCN are summarized in table \ref{result}. Note that some datasets such as Polblogs dataset do not have node attributes so we assign one-hot vectors as their attributes. They are not applicable for the attack based on attribute similarity. As can be observed in table \ref{result}, GraphMI achieves the best performance across nearly all the datasets, which demonstrates the effectiveness of GraphMI. For instance, GraphMI achieves an average attack AUC of 86.8$\%$ and an average attack AP of 88.3$\%$ on the Cora dataset, which exceeds the best baseline method by at least 5 percent points. One exception is Citesser where the attack performance of GraphMI is relatively lower than attribute similarity, which could be explained by more abundant node attribute information of Citeseer compared with other datasets. Thus using node attribute similarity alone could achieve good performance on the Citesser dataset.

In figure \ref{gnn type}, we show the attack performance of GraphMI on three GNNs. In general, GraphMI can achieve success on all three GNNs. Specially, we observe that GraphMI has better attack performance on the GAT model. This may be explained by the fact that GAT model is more powerful at node classification \cite{velivckovic2017graph}. Thus, it is able to build a stronger correlation between graph topology and node labels. GraphMI can take advantage of such a stronger correlation and achieve better attack performance. \par 
In figure \ref{edge influence}(a), we present the influence of node label proportion on attack performance. As can be observed from the plot, with more node labels, the attack performance will increase gradually. This is because more node labels offer more information and regularization (Equ. \ref{equ1}), which the attacker can exploit to infer the training graph. Moreover, GraphMI is robust to the number of node labels. For example, GraphMI can still achieve over 80 $\%$ AUC and AP when only 20$\%$ node labels are available.

\begin{figure}[t]
	\centering
	\subfigure[]{\includegraphics[width=0.48\linewidth]{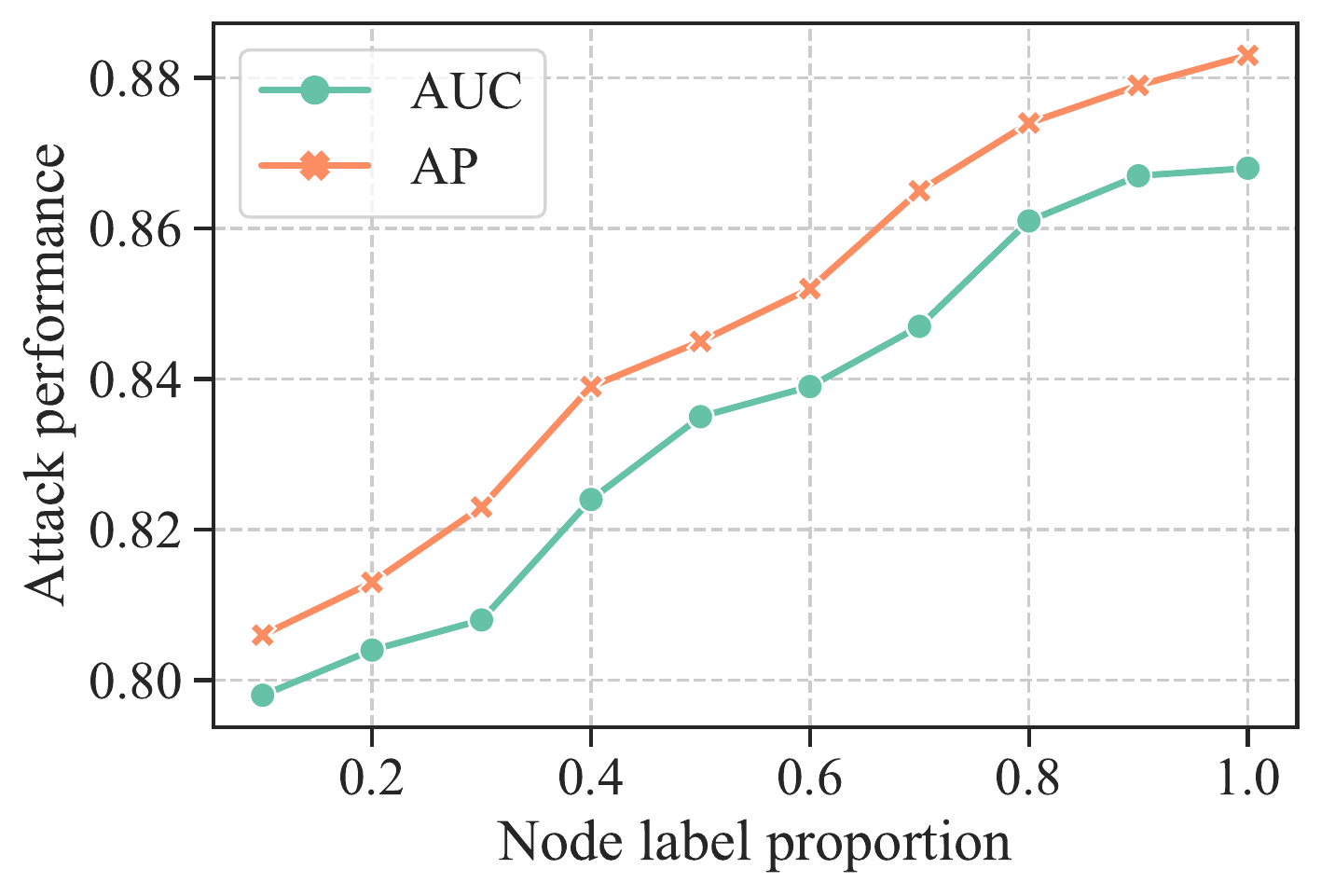}}
	\subfigure[]{\includegraphics[width=0.48\linewidth]{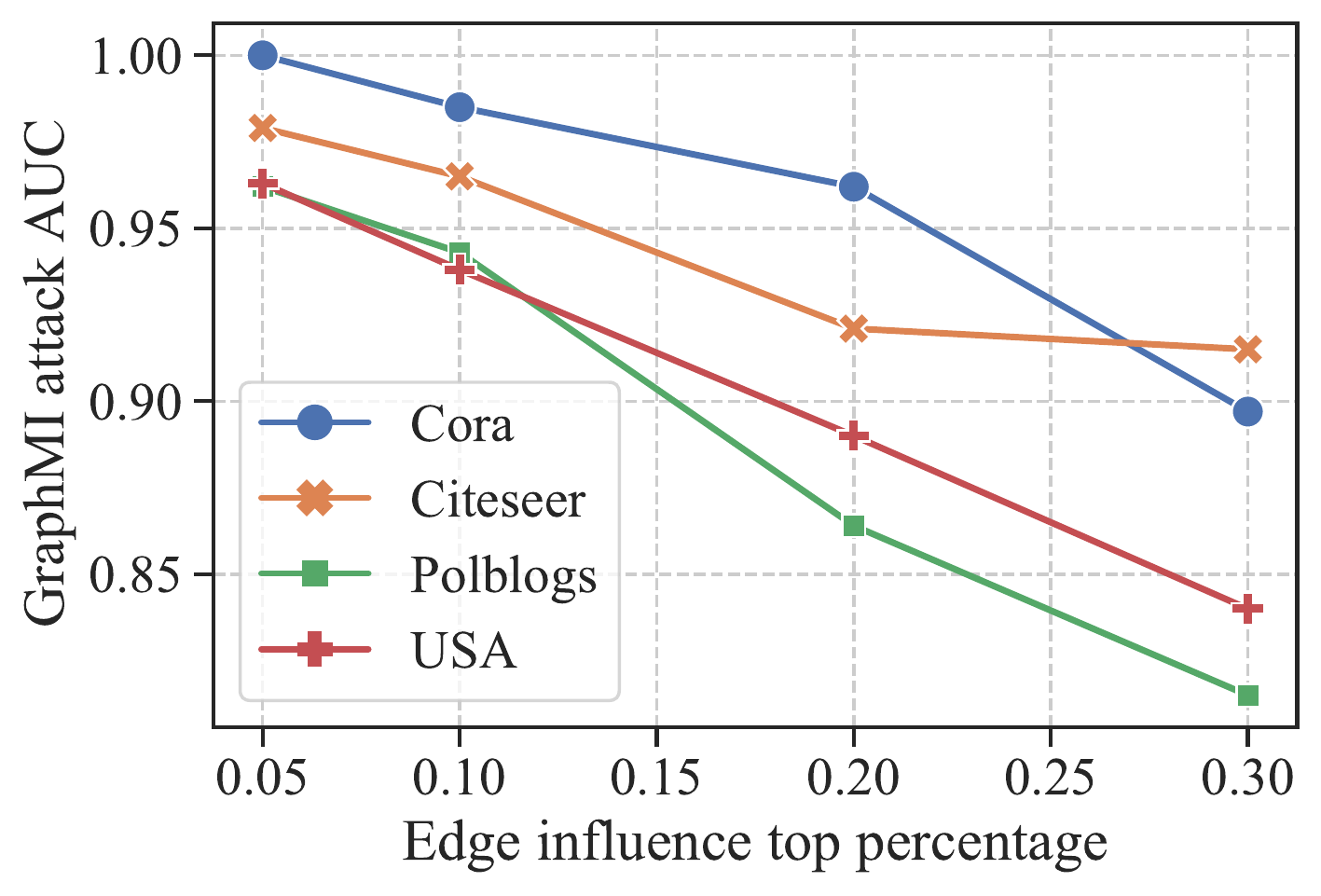}}
	\caption{Impact of node label proportion (a) and edge influence (b) on the Cora dataset.}
	\label{edge influence}
	\vspace{-0.5cm}
\end{figure}

\subsubsection{Graph Isomorphism and Macro-level Graph Statistics}
The aforementioned results show that GraphMI can reconstruct the training graph well at the edge level by measuring AUC or AP. Further, we are interested in how the reconstructed graph resembles the training graph at the graph level. Here we use graph isomorphism and macro-level graph statistics to show the graph-level similarity.\par
\emph{Graph isomorphism} compares
the structure of the reconstructed graph with the original graph and determines their similarity. The graph isomorphism problem is well-known to be intractable in polynomial time; thus, approximate algorithms such as \emph{Weisfeiler-Lehman (WL)} algorithm \cite{shervashidze2011weisfeiler} are widely used for addressing it. We use the open-source implementation in experiments to calculate the Weisfeiler-Lehman kernel\footnote{https://github.com/emanuele/jstsp2015}.\par
We further explore whether the reconstructed graph and the original graph have similar graph statistics by analyzing four widely
used graph statistics: Degree distribution, local clustering
coefficient (LCC), betweenness centrality (BC), and closeness centrality (CC). We refer the readers to Appendix
for detailed descriptions of these statistics. In experiments, we use networkx \footnote{https://github.com/networkx/networkx} for the calculation of LCC, CC, and BC, and bucketize the statistic domain into 10 bins to measure their distributions by cosine similarity. \par
Table \ref{iso} illustrates the attack performance in terms of graph isomorphism and macro-level graph statistics measured by cosine similarity. Note that we need to use the random sampling module to sample discrete adjacency matrices for the calculation of Weisfeiler-Lehman kernel and macro-level graph statistics.
Generally, our attack
achieves strong performance across all the tests. For instance, the WL graph
kernel on Citeseer achieves 94.6$\%$. Besides, the
cosine similarity of the local clustering coefficient and betweenness centrality distribution
is larger than 98.0$\%$ for all the datasets. For degree distribution
and closeness centrality distribution, the attack performance
is slightly worse; however, we can still achieve cosine similarity larger than 75.0$\%$.

\subsubsection{Edge Influence.}We do experiments to verify our claim that edges with greater influence are more likely to be inferred successfully through model inversion attack (Theorem \ref{theorem1}). Note that it will be very time-consuming to measure the influence of each edge exactly. According to Equ. \ref{edge}, removing edges with greater influence will cause a greater drop of prediction accuracy. To select edges with great influence, we apply the state-of-the-art topology attacks \cite{xu2019topology, zugner2019adversarial, li2021adversarial} on graphs by removing edges and average the edge influence ranks. In Figure \ref{edge influence}(b), we show that for edges with top $5\%$ influence GraphMI achieves the attack AUC of nearly 1.00 in Cora dataset. This implies that the privacy leakage will be more severe if sensitive edges are those with greater influence.

\begin{figure}[t]
	\centering
	\subfigure[]{\includegraphics[width=0.48\linewidth]{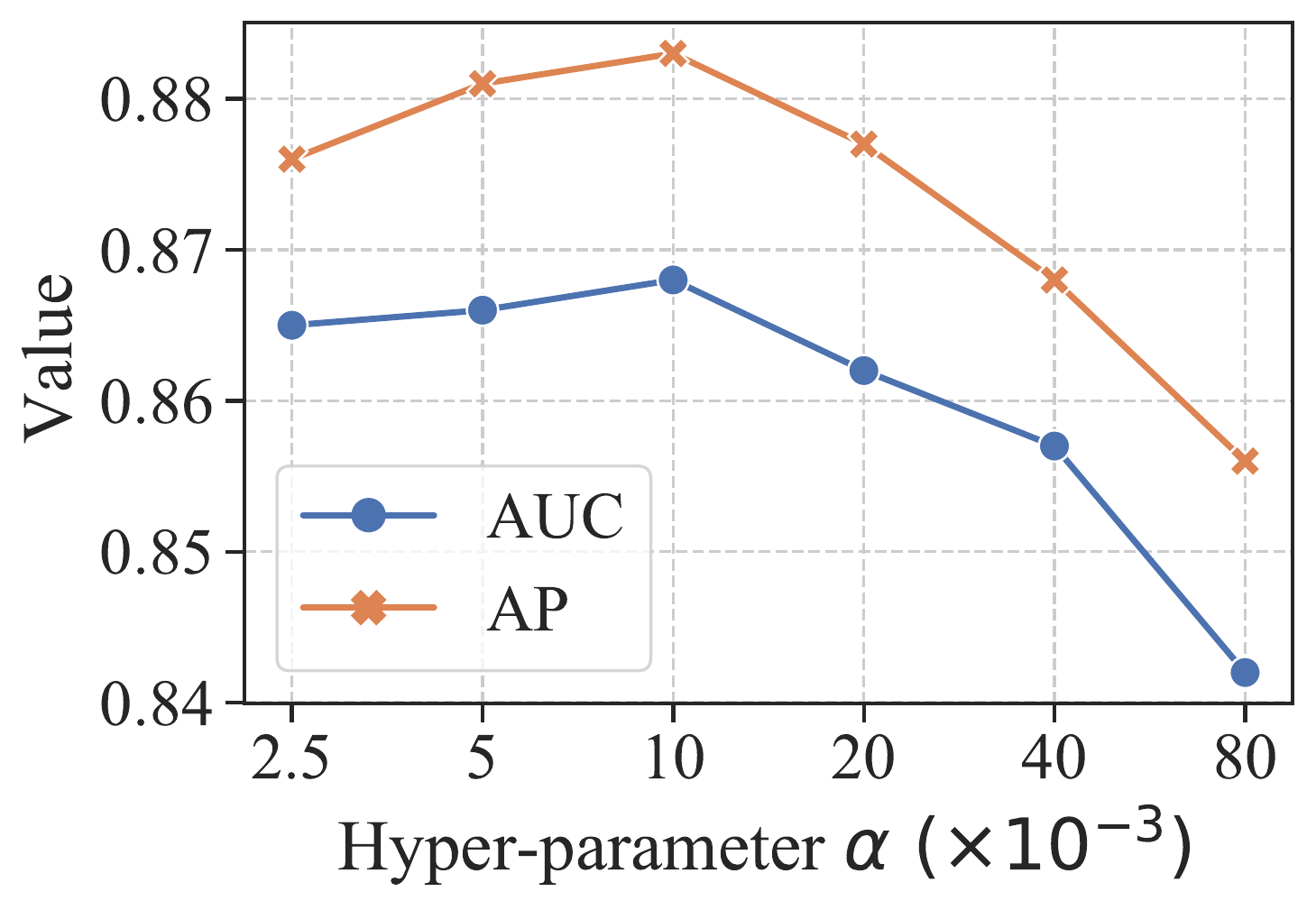}}
    \subfigure[]{\includegraphics[width=0.48\linewidth]{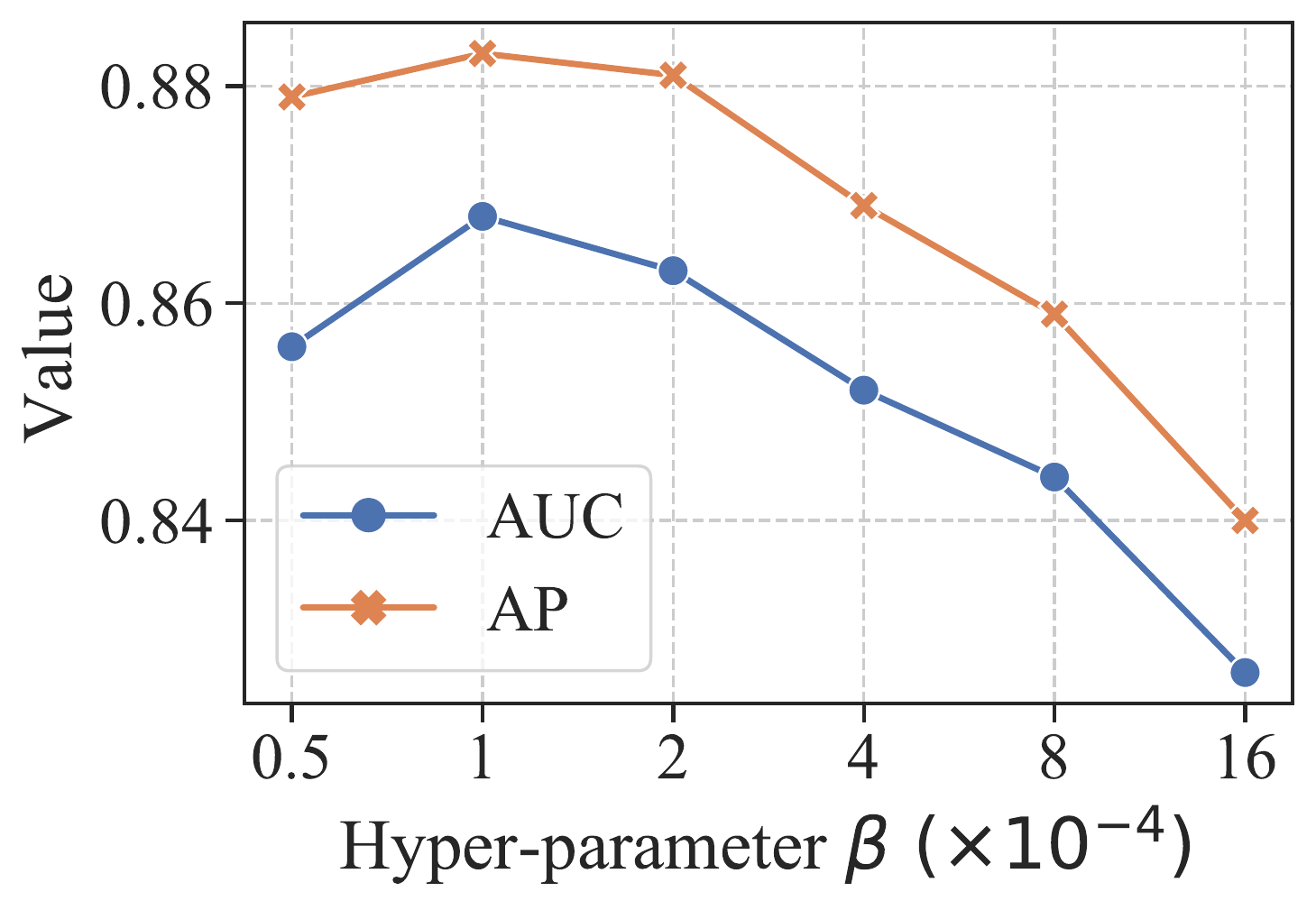}}\\
    \subfigure[]{\includegraphics[width=0.48\linewidth]{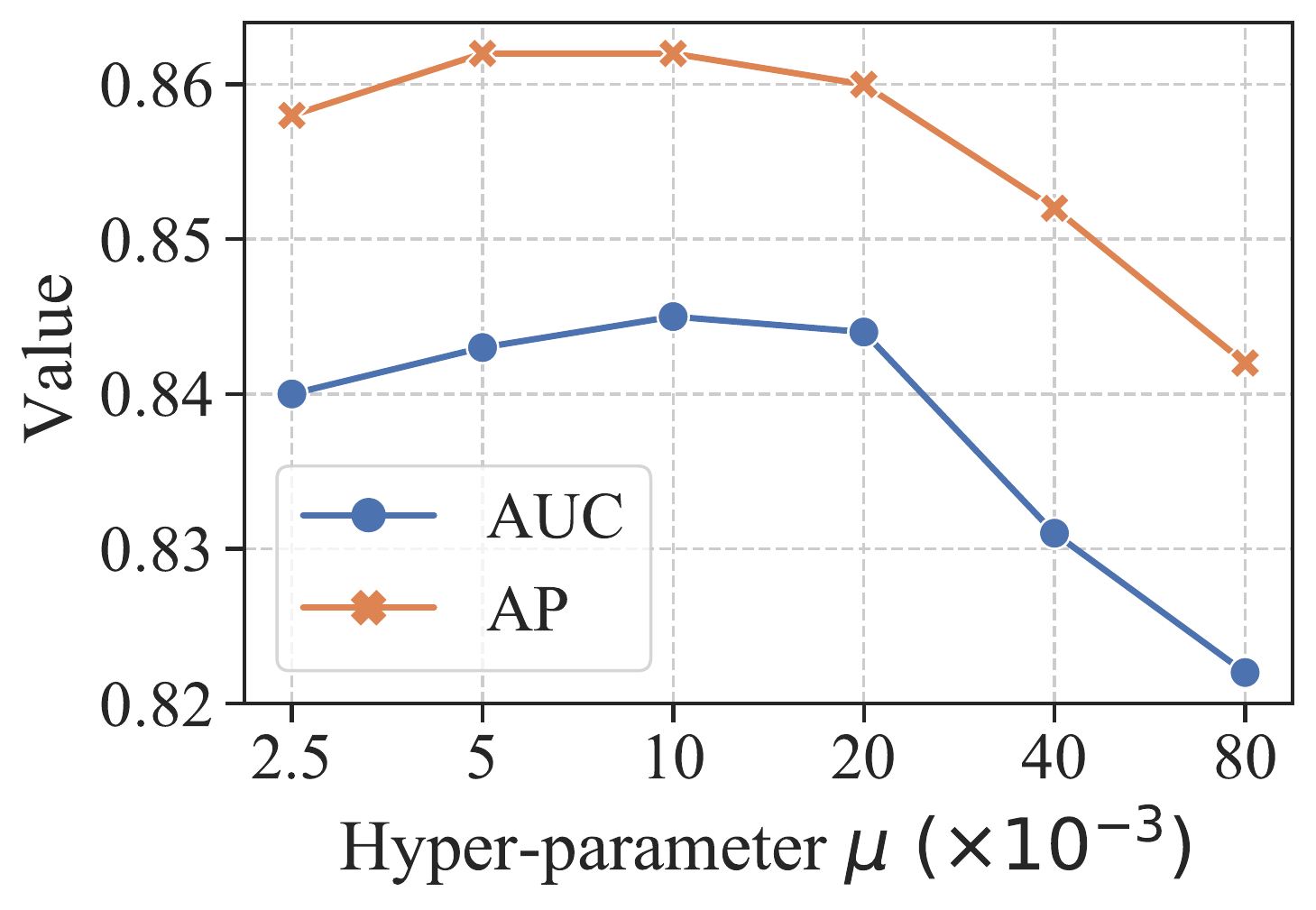}}
    \subfigure[]{\includegraphics[width=0.48\linewidth]{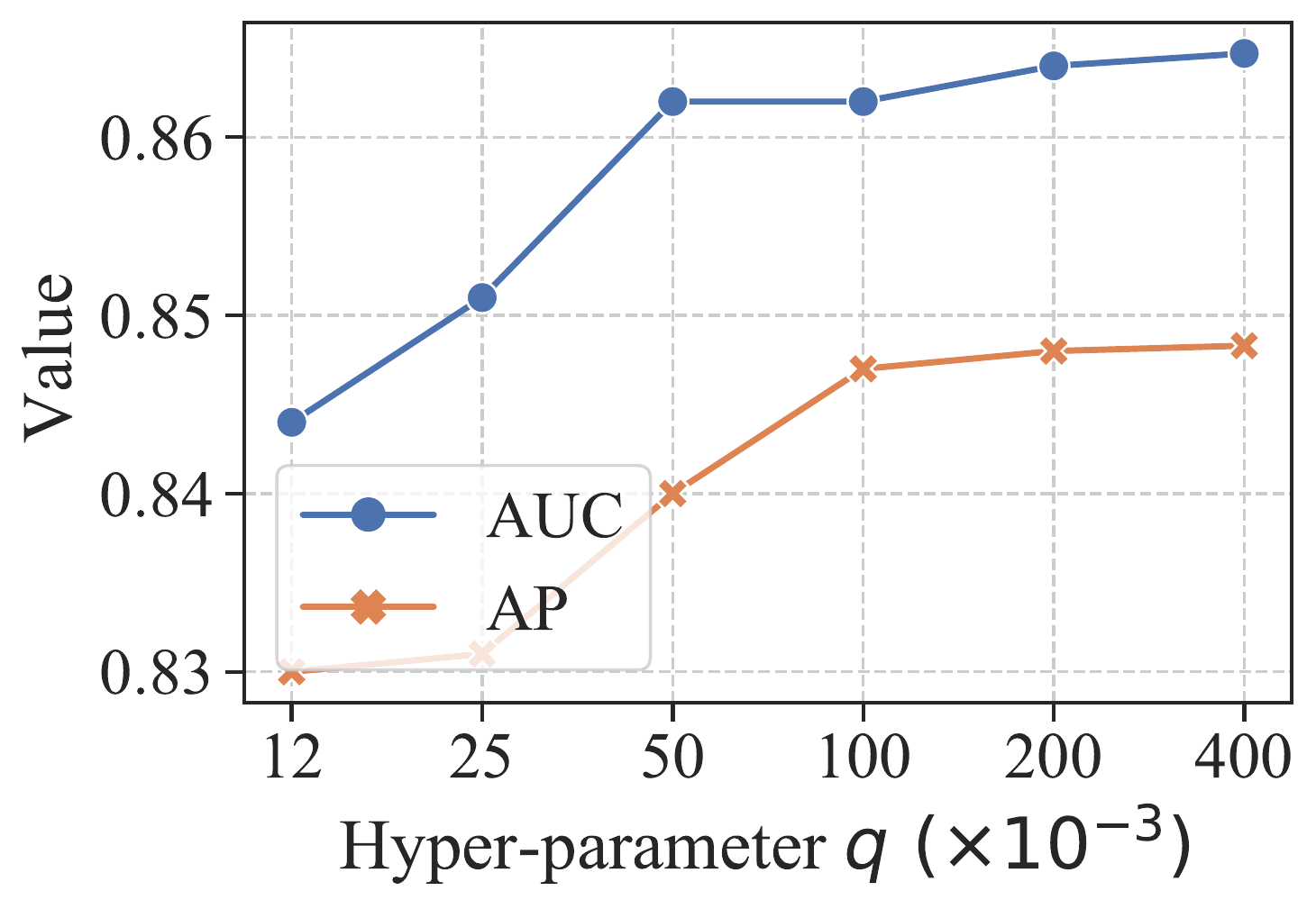}}\\
    \subfigure[]{\includegraphics[width=0.5\linewidth]{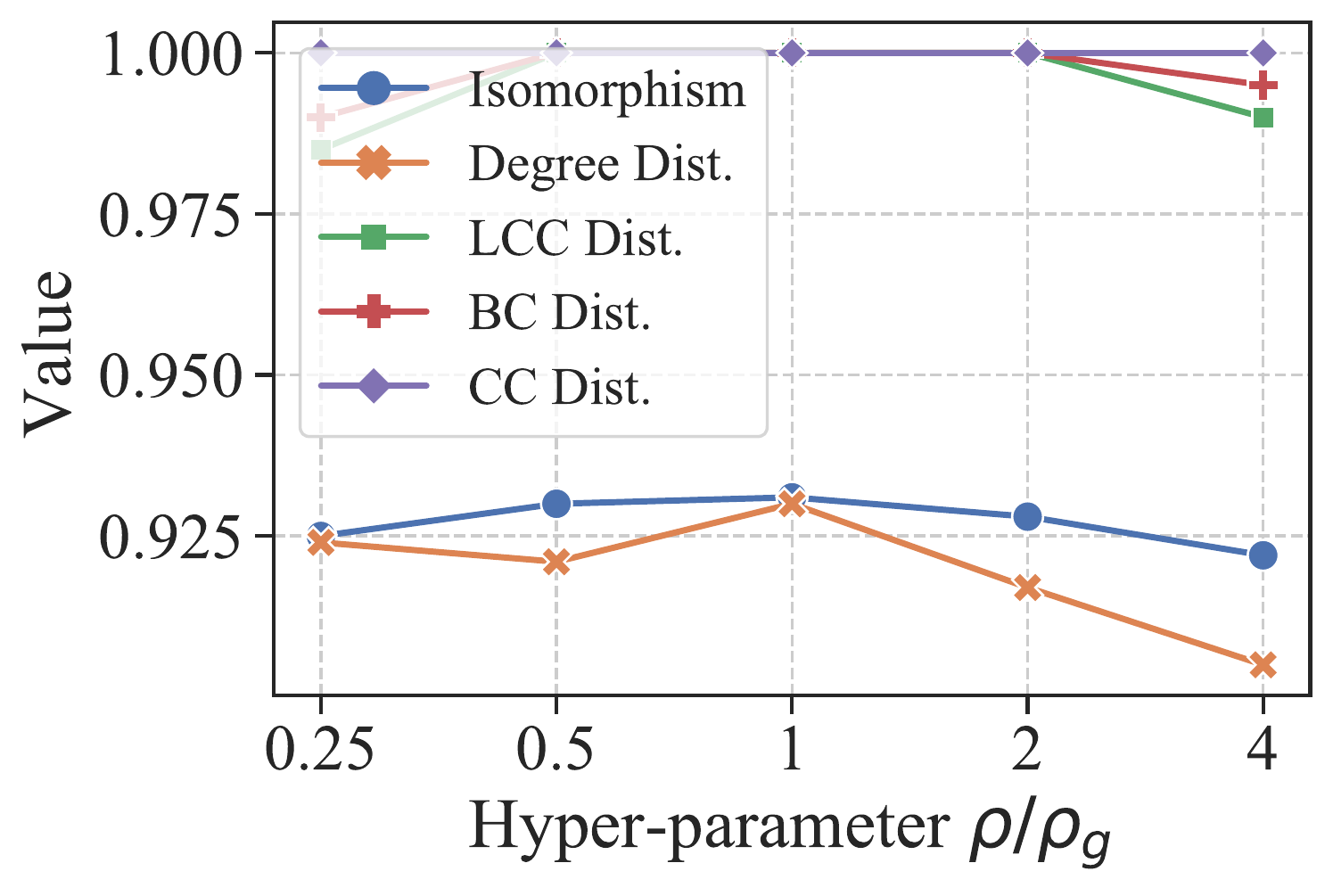}}
	\caption{ Results of hyper-parameter analysis on Cora dataset}
	\label{param analysis}
\end{figure}

\subsubsection{Effectiveness of Black-box attack}
\begin{table}[t]
\centering
  \fontsize{8}{11}\selectfont
  \caption{Performance of Black-box attacks. We report the attack AUC ($\%$), AP ($\%$), AQ, and AT of two black-box attack methods. We also show the results of white-box GraphMI for reference.
  } 
  \label{black box results}
    \begin{tabular}{c|c|cccc}
    \toprule
    \multirow{1}{*}{Dataset}
    &Method&AUC&AP&AQ&AT (s) \cr
    \hline
    \multirow{3}{*}{Cora}&GraphMI&86.8&88.3&-&29.0\cr
    &GE&84.5&86.2&10,000&760.9\cr
    &RL-GraphMI&80.5&79.2&270,800&5870.6\cr
 
    \cline{1-6}
    \multirow{3}{*}{Citeseer}&
    GraphMI&87.8&88.5&-&30.1\cr
    &GE&86.5&87.3&10,000&787.6\cr
    &RL-GraphMI&87.5&88.2&473,200&6082.7\cr
   
    \cline{1-6}
    \multirow{3}{*}{Polblogs}&
    GraphMI&79.3&79.7&-&19.8\cr
    &GE&69.0&66.4&10,000&671.2\cr
    &RL-GraphMI&54.6&51.2&1,902,500&19102.5\cr
    
        \cline{1-6}
    \multirow{3}{*}{USA}&
    GraphMI&80.6&81.3&-&28.4\cr
    &GE&66.7&66.9&10,000&780.4\cr
    &RL-GraphMI&55.2&53.1&1,359,900&11044.1\cr
    
        \cline{1-6}
    \multirow{3}{*}{Brazil}&
    GraphMI&86.6&88.8&-&12.5\cr
    &GE&78.2&79.3&10,000&571.6\cr
    &RL-GraphMI&74.9&76.3&103,800&3389.3\cr

    \bottomrule
    \end{tabular}
    \label{black-box result}
\end{table}
In Table \ref{black-box result}, we show the attack performance of Gradient Estimation (abbreviated as GE) and RL-GraphMI. We also list the results of GraphMI for reference. Generally, GA and RL-GraphMI can achieve competitive attack performance in the black-box setting. For example, GE and RL-GraphMI achieve 86.5$\%$ and 87.5$\%$ attack AUC on Citesser, which are only 1.3 and 0.3 percent points less than the white box attack GraphMI. We also observe that RL-GraphMI has better performance on datasets with fewer nodes and edges such as Brazil, Cora, and Citesser because the agent has a smaller searching space. Compared with RL-GraphMI, gradient estimation has less average time (AT) and average queries (AQ), which makes it a more efficient choice for the black-box attack.\par

Theorem \ref{convergence therom} analyzes the convergence property of the gradient estimation method. Here, we conduct experiments to verify the theorem of convergence. Fig. \ref{convergence plot} shows that the L2 norm estimated gradient converges to 0 with respect to iterations. Therefore, Theorem \ref{convergence therom} is verified empirically. 
\begin{figure}[t]
	\centering
    \includegraphics[width=0.6\linewidth]{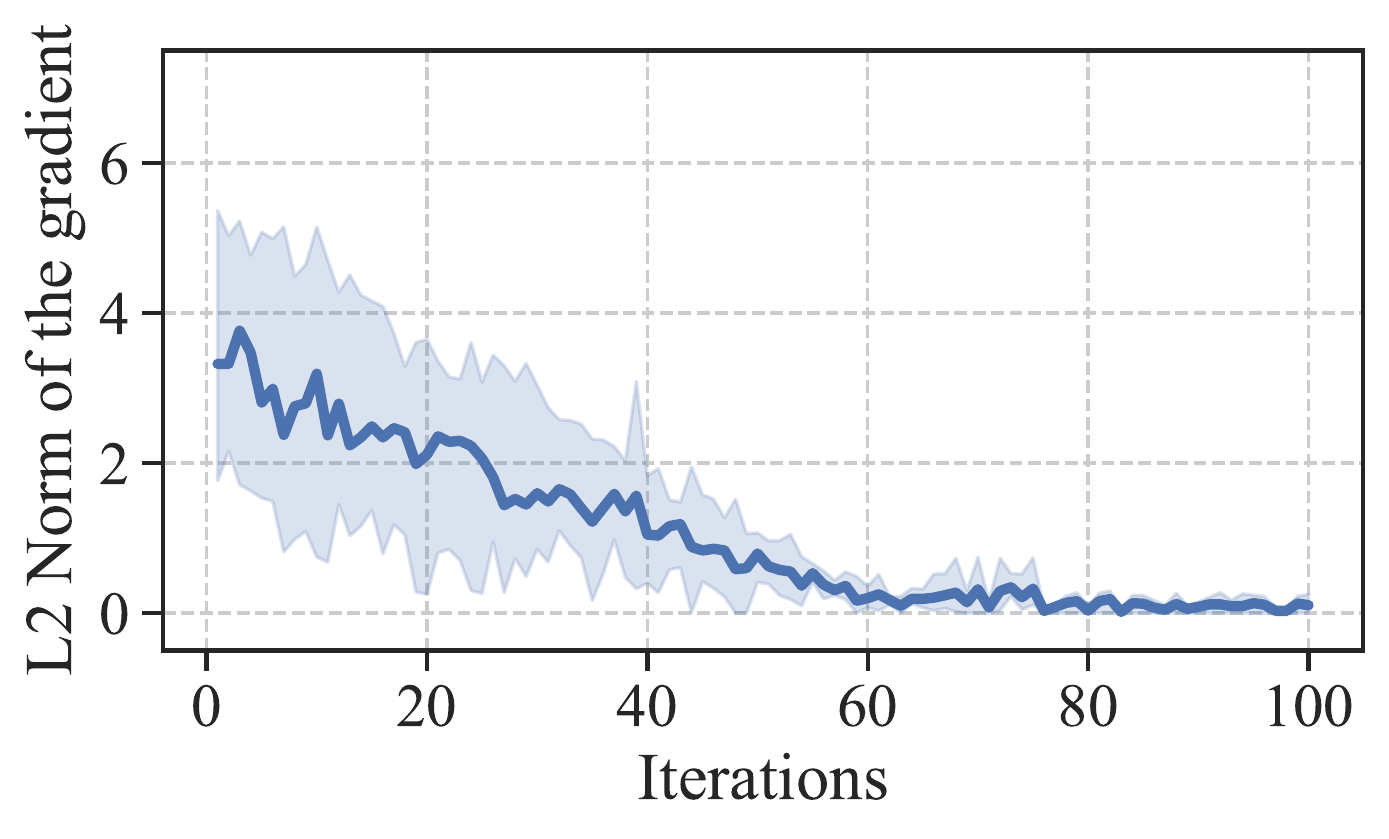}
	\caption{Convergence of the gradient estimation. We show the means and $95\%$ confidence intervals.}
	\label{convergence plot}
\end{figure}
\subsubsection{Ablation Study and Hyper-parameter Analysis}
 \begin{table}[t]
\centering
  \fontsize{8}{11}\selectfont
  \caption{Ablation study of auto-encoder module ($\%$).} 
  \vspace{-3mm}
    \begin{tabular}{c|c|cc}
    \toprule
    \multirow{1}{*}{Dataset}&
    w/ auto-encoder module&AUC&AP \\
    \hline
    \multirow{2}{*}{Cora}&Yes&\textbf{86.8}&\textbf{88.3}\cr
    &No&82.5&81.7\cr
 
    \cline{1-4}
    \multirow{2}{*}{Citeseer}&
    Yes&\textbf{87.8}&\textbf{88.5}\\
    &No&85.6&86.2\\
   
    \cline{1-4}
    \multirow{2}{*}{Polblogs}&
    Yes&\textbf{79.3}&\textbf{79.7}\\
    &No&70.1&70.4\\

    \bottomrule
    \end{tabular}
\label{ablation study}
\end{table}

Here we conduct ablation studies to show the impact of the auto-encoder module on the attack performance in Table \ref{ablation study}. We can observe that with the auto-encoder module, GraphMI achieves the best results over all three datasets. This is because the node embeddings generated by the graph auto-encoder module fully exploit the information of node attributes, graph topology, and the target model. The similarity of such node embeddings could indicate the existence of edges more precisely.\par
We also explore the sensitivity of hyper-parameters $\alpha$, $\beta$, $\mu$, $q$, and $\rho$. In Fig~\ref{param analysis}, we alter the value of
hyperparameters to see how they affect the attack performance. Specifically, we vary $\alpha$ from 0.00025 to 0.008, $\beta$ from 0.00005 to 0.0016, $\mu$ from 0.0025 to 0.08, $q$ from 12 to 400, and $\rho/\rho_g$ from 0.25 to 4 in a log scale of base 2. As we can observe, the attack performance of GraphMI can be boosted by choosing proper values for $\alpha$ and $\beta$. For the hyperparameters in the black-box attack, gradient estimation has the optimal performance at $\mu = 0.01$. This may be explained by the fact that too large step leads to biased gradient estimation while too small step incurs relatively more random noise. Gradient estimation can achieve better performance with larger $q$ due to more accurate estimation. Considering the efficiency, we choose $q=100$ as the default setting in experiments. In Fig~\ref{param analysis}(e), we show how the influence of estimated graph density $\rho$ and $\rho_g$ is the ground truth density. We can observe that our attack is generally robust to $\rho$: GraphMI achieves over 0.90 on all metrics with different estimated graph density $\rho$. What's more, the optimal performance of GraphMI is achieved when $\rho= \rho_g$.

\section{Defending against GraphMI}
In this section, we introduce and evaluate two defense methods against GraphMI: one is to train GNNs with differential privacy and the other is to preprocess the graph by perturbing edges.
\subsection{Defense Performance of Differentail Privacy}
Differential privacy (DP) is one general approach for protecting privacy. More formally:
\begin{mydef}
    (($\epsilon,\delta$)-differential privacy) A randomized mechanism $M$ with domain $\mathcal{R}$ and output $\mathcal{S}$ satisfies ($\epsilon,\delta$)-differential privacy if for any neighboring  datasets $D, D' \in \mathcal{R}$, which differ by at most one record and for any subsets of outputs $S$ it holds that:
\begin{equation}
    Pr[M(D)\in S]\le e^{\epsilon}Pr[M(D')\in S]+\delta,
\end{equation}
where $\epsilon$ is the privacy budget and $\delta$ denotes the failure probability.
\end{mydef}
Here, we investigate the impact of differential privacy on three attacks. $(\epsilon,\delta)$ -
DP is ensured by adding Gaussian noise to the clipped gradients in each training iteration \cite{abadi2016deep}. In experiments, $\delta$ is set to $10^{-5}$ and the noise scale is varied to obtain target GNN models with different $\epsilon$ from 1.0 to 10.0.
The GraphMI attack performance and their model utility are presented in Table \ref{DP}. As the privacy budget $\epsilon$ drops, the performance of GraphMI attack deteriorates at the price of a huge utility drop. For example, the trained GNN model only retains a node classification accuracy of 0.48 when the GraphMI attack AUC is decreased to 0.60. Generally, enforcing DP on target models cannot prevent GraphMI attack while preserving decent model utility.
\subsection{Defense Performance of Graph Preprocessing}
Here we propose one new defense method based on graph preprocessing. Instead of adding noise in GNN model training, our method pre-processes the training graph to "hide" real edges from model inversion. Specifically, we apply the following three strategies in training graph pre-processing:

\textbf{Randomly rewiring: }For the exisiting edge ($v_i, v_j$) in the training graph, with probability $p$, the rewiring operation deletes the
existing edge between nodes $v_i$ and $v_j$, while adding an edge to connect nodes $v_i$ and $v_k$ or $v_j$ and $v_k$ ($k \neq i$ or $j$).\par
\textbf{Adding new edges: }Specifically, to preserve utility of trained GNN model, we add pseudo edges with high feature similarity. Compared with convolutional neural networks for image data, graph neural network generally has fewer layers and parameters. GNN models essentially aggregate features according to graph structure. Previous work on graph adversarial learning~\cite{wu2019adversarial} suggests that adding edges with low feature similarity will severely undermines the utility of GNN model in node classification. Compared with adding edges randomly in the training graph, adding edges with high feature similarity helps preserve model utility. \par
\textbf{Randomly flipping edges: }For each entry in the adjacency matrix, with probability $1-p~(p \in (0,1])$ we keep it the same as the original graph; with probability $p$, we sample it randomly from $Bern(1/2)$. $Bern(\cdot)$ is the Bernoulli distribution. Such perturbation is proven to be $\epsilon$-edge level differential privacy ($\epsilon = ln(\frac{2}{p}-1)$) \cite{wu2021linkteller}.
\begin{table}[t]
\centering
\caption{The performance of the GraphMI, GE and RL-GraphMI attack (AUC) against GCN trained with differential privacy on Cora dataset}
\begin{tabular}{lcccc}
\toprule
Method & ACC & GraphMI&GE& RL-GraphMI\\ \midrule
$\epsilon=1.0 $ &                0.48             &         0.60&0.61&0.58       \\ 
$\epsilon=5.0 $  &               0.65              &        0.72&0.70&0.65        \\ 
$\epsilon=10.0 $ &               0.78               &         0.84&0.79&0.76       \\ 
no DP   &                        0.80 &0.87&0.85&0.81        \\ \bottomrule
\end{tabular}
\label{DP}
\end{table}
\begin{figure}[t]
	\centering
    \subfigure[GraphMI]{\includegraphics[width=0.32\linewidth]{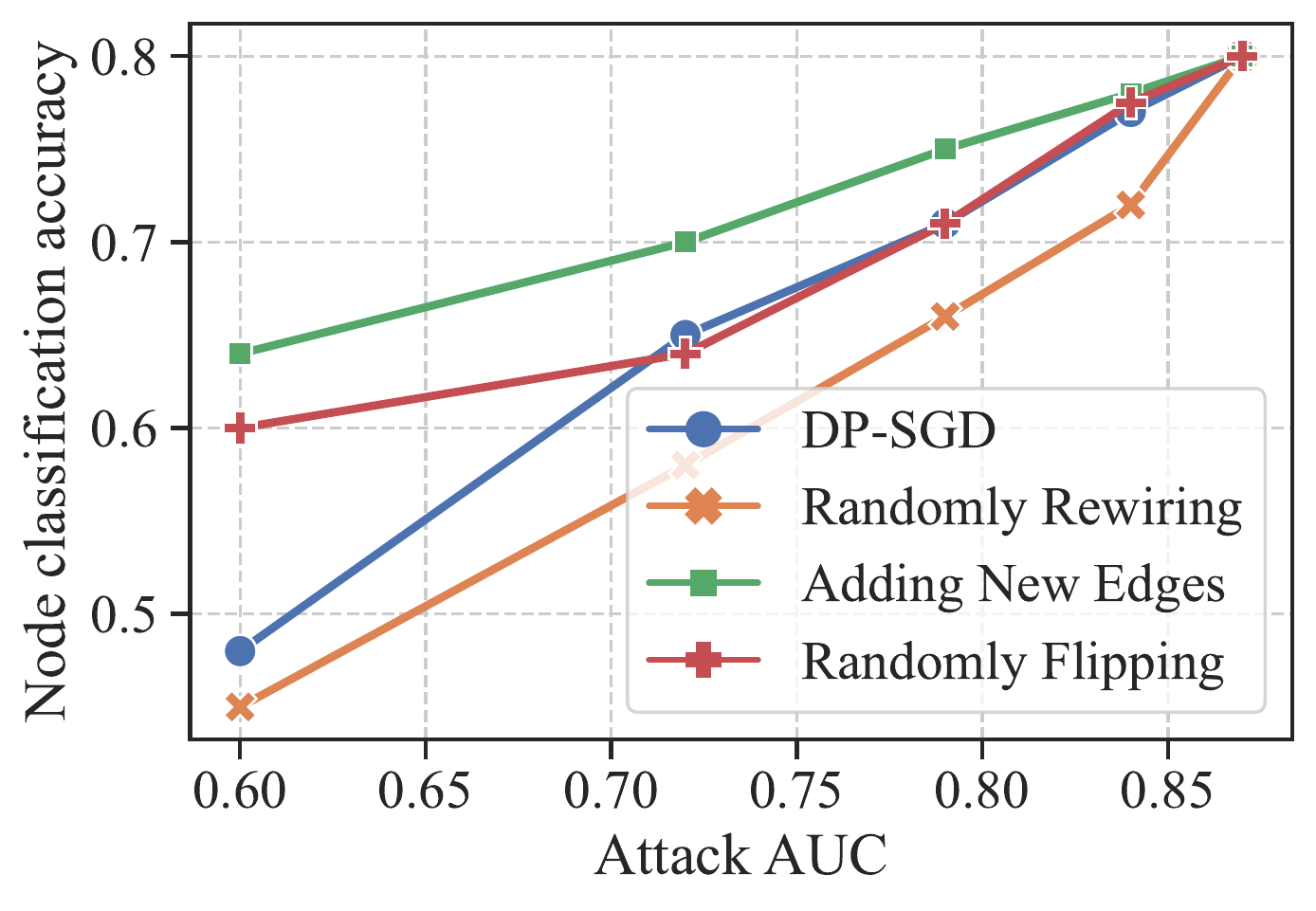}}
    \subfigure[GE]{\includegraphics[width=0.32\linewidth]{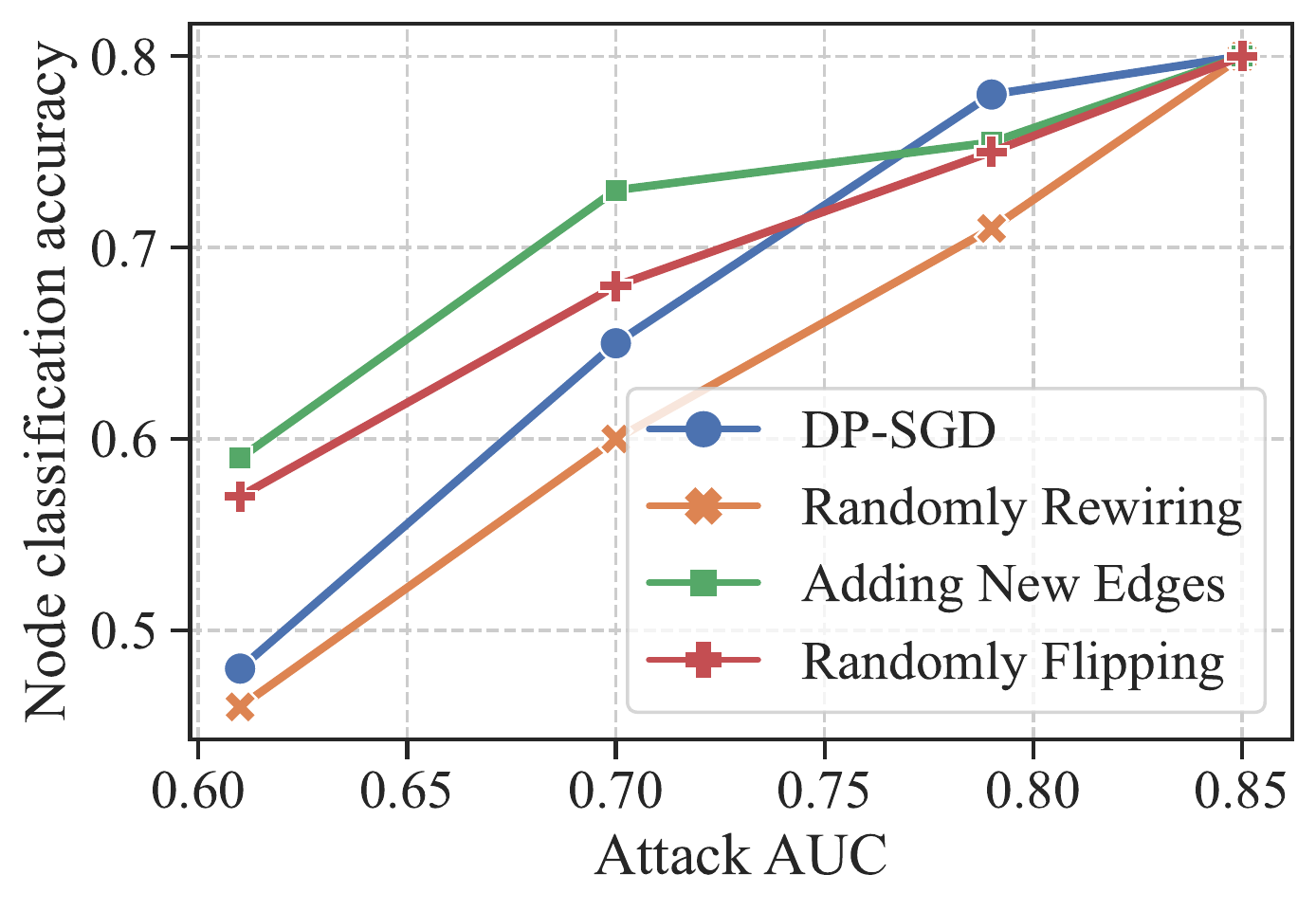}}
    \subfigure[RL-GraphMI]{\includegraphics[width=0.32\linewidth]{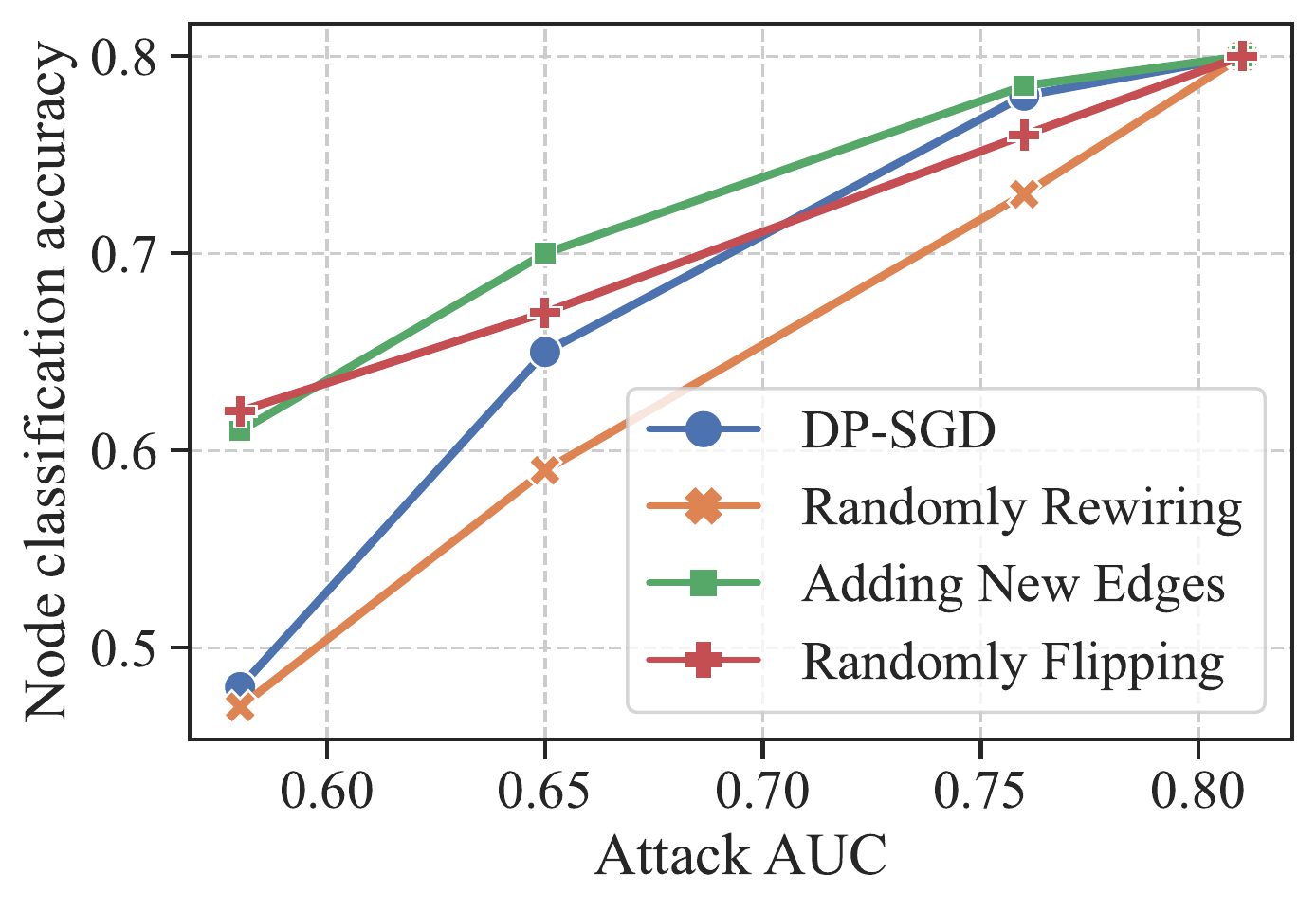}}
	\caption{The trade-off curves between attack AUC and node classification accuracy on Cora dataset.}
	\label{graph preprocessing}
\end{figure}

Figure \ref{graph preprocessing} shows the trade-off between GNN utility and attack AUC of GraphMI, GE, and RL-GraphMI. Four defense strategies are shown, including DP-SGD and three graph pre-processing methods. Generally, all these defense methods can decrease the attack AUC at the price of node classification accuracy. Among the three pre-processing strategies, adding edges has the best defense performance. For instance, the trained GNN model has a classification accuracy of 0.64 when the attack AUC is decreased to 0.60 for GraphMI. This verifies our claim that adding edges with high feature similarity helps preserve model utility while defending against GraphMI attacks. On the other hand, randomly rewiring edges has a poor performance. The rewiring operation may link nodes with distinct labels or dissimilar features, which may damage the GNN model utility a lot. 
\subsection{Discussion}
The defense results in Section 7.1 and Section 7.2 indicate that our graph model inversion attack is still effective even with differential privacy or graph preprocessing. There are several potential ways to strengthen the defense such as adversarial training and detection of adversarial queries in the black-box setting. We left them for future work.
\section{Conclusion}
In this paper, we conducted a systematic study on model inversion attacks against graph neural networks.
We first presented GraphMI, a white-box attack specifically designed and optimized for extracting private graph-structured data from GNNs.
Though GraphMI could effectively reconstruct the edges by exploiting the properties of graph and GNNs, it is infeasible in the black-box setting where the attacker is only allowed to query the GNN API and receive the classification results.
Therefore, we extended GraphMI to the hard-label black-box setting and proposed two black-box attacks based on gradient estimation and reinforcement learning.
Extensive experimental results showed its effectiveness on several state-of-the-art graph neural networks.
Finally, we showed that imposing differential privacy on graph neural networks and graph preprocessing are not effective enough to protect privacy while preserving decent model utility.

This paper provided potential tools for investigating the privacy risks of deep learning models on graph-structured data.
Interesting future directions include: 1) Design black-box attacks with better query efficiency; 2) Design countermeasures with a better trade-off between utility and privacy.
\vspace{-0.5cm}
\section*{Acknowledgment}
This research was partially supported by grants from the National Natural Science Foundation of China (Grants No.61922073 and U20A20229), and the Fundamental Research Funds for the Central Universities (Grant No.WK2150110021).

\bibliographystyle{abbrv}
\bibliography{main}

\begin{thebibliography}{10}

\bibitem{abadi2016deep}
M.~Abadi, A.~Chu, I.~Goodfellow, H.~B. McMahan, I.~Mironov, K.~Talwar, and
  L.~Zhang.
\newblock Deep learning with differential privacy.
\newblock In {\em CCS 2016}, pages 308--318, 2016.

\bibitem{adamic2005political}
L.~A. Adamic and N.~Glance.
\newblock The political blogosphere and the 2004 us election: divided they
  blog.
\newblock In {\em Proceedings of the 3rd international workshop on Link
  discovery}, pages 36--43, 2005.

\bibitem{aivodji2019gamin}
U.~A{\"\i}vodji, S.~Gambs, and T.~Ther.
\newblock Gamin: An adversarial approach to black-box model inversion.
\newblock {\em arXiv preprint arXiv:1909.11835}, 2019.

\bibitem{blocki2012johnson}
J.~Blocki, A.~Blum, A.~Datta, and O.~Sheffet.
\newblock The johnson-lindenstrauss transform itself preserves differential
  privacy.
\newblock In {\em 2012 IEEE 53rd Annual Symposium on Foundations of Computer
  Science}, pages 410--419. IEEE, 2012.

\bibitem{blocki2013differentially}
J.~Blocki, A.~Blum, A.~Datta, and O.~Sheffet.
\newblock Differentially private data analysis of social networks via
  restricted sensitivity.
\newblock In {\em Proceedings of the 4th conference on Innovations in
  Theoretical Computer Science}, pages 87--96, 2013.

\bibitem{bojchevski2019certifiable}
A.~Bojchevski and S.~G{\"u}nnemann.
\newblock Certifiable robustness to graph perturbations.
\newblock {\em NeurIPS}, 2019.

\bibitem{borgwardt2005protein}
K.~M. Borgwardt, C.~S. Ong, S.~Sch{\"o}nauer, S.~Vishwanathan, A.~J. Smola, and
  H.-P. Kriegel.
\newblock Protein function prediction via graph kernels.
\newblock {\em Bioinformatics}, 21(suppl\_1):i47--i56, 2005.

\bibitem{chang2020restricted}
H.~Chang, Y.~Rong, T.~Xu, W.~Huang, H.~Zhang, P.~Cui, W.~Zhu, and J.~Huang.
\newblock A restricted black-box adversarial framework towards attacking graph
  embedding models.
\newblock In {\em AAAI}, volume~34, pages 3389--3396, 2020.

\bibitem{chen2020link}
J.~Chen, X.~Lin, Z.~Shi, and Y.~Liu.
\newblock Link prediction adversarial attack via iterative gradient attack.
\newblock {\em IEEE Transactions on Computational Social Systems},
  7(4):1081--1094, 2020.

\bibitem{chen2018fast}
J.~Chen, Y.~Wu, X.~Xu, Y.~Chen, H.~Zheng, and Q.~Xuan.
\newblock Fast gradient attack on network embedding.
\newblock {\em arXiv preprint arXiv:1809.02797}, 2018.

\bibitem{dai2018adversarial}
H.~Dai, H.~Li, T.~Tian, X.~Huang, L.~Wang, J.~Zhu, and L.~Song.
\newblock Adversarial attack on graph structured data.
\newblock In {\em ICML}, pages 1115--1124, 2018.

\bibitem{dai2019adversarial}
Q.~Dai, X.~Shen, L.~Zhang, Q.~Li, and D.~Wang.
\newblock Adversarial training methods for network embedding.
\newblock In {\em WWW}, pages 329--339, 2019.

\bibitem{duddu2020quantifying}
V.~Duddu, A.~Boutet, and V.~Shejwalkar.
\newblock Quantifying privacy leakage in graph embedding.
\newblock {\em arXiv preprint arXiv:2010.00906}, 2020.

\bibitem{dwork2014algorithmic}
C.~Dwork, A.~Roth, et~al.
\newblock The algorithmic foundations of differential privacy.
\newblock {\em Found. Trends Theor. Comput. Sci.}, 9(3-4):211--407, 2014.

\bibitem{fan2020graph}
W.~Fan, Y.~Ma, Q.~Li, J.~Wang, G.~Cai, J.~Tang, and D.~Yin.
\newblock A graph neural network framework for social recommendations.
\newblock {\em TKDE}, 2020.

\bibitem{feng2019graph}
F.~Feng, X.~He, J.~Tang, and T.-S. Chua.
\newblock Graph adversarial training: Dynamically regularizing based on graph
  structure.
\newblock {\em TKDE}, 2019.

\bibitem{flaxman2004online}
A.~D. Flaxman, A.~T. Kalai, and H.~B. McMahan.
\newblock Online convex optimization in the bandit setting: gradient descent
  without a gradient.
\newblock 2005.

\bibitem{fredrikson2015model}
M.~Fredrikson, S.~Jha, and T.~Ristenpart.
\newblock Model inversion attacks that exploit confidence information and basic
  countermeasures.
\newblock In {\em SIGSAC}, pages 1322--1333, 2015.

\bibitem{fredrikson2014privacy}
M.~Fredrikson, E.~Lantz, S.~Jha, S.~Lin, D.~Page, and T.~Ristenpart.
\newblock Privacy in pharmacogenetics: An end-to-end case study of personalized
  warfarin dosing.
\newblock In {\em USENIX Security}, pages 17--32, 2014.

\bibitem{gao2018information}
X.~Gao, B.~Jiang, and S.~Zhang.
\newblock On the information-adaptive variants of the admm: an iteration
  complexity perspective.
\newblock {\em Journal of Scientific Computing}, 76(1):327--363, 2018.

\bibitem{pmlr-v70-gilmer17a}
J.~Gilmer, S.~S. Schoenholz, P.~F. Riley, O.~Vinyals, and G.~E. Dahl.
\newblock Neural message passing for quantum chemistry.
\newblock In {\em ICML}, pages 1263--1272, 2017.

\bibitem{hamilton2017inductive}
W.~Hamilton, Z.~Ying, and J.~Leskovec.
\newblock Inductive representation learning on large graphs.
\newblock In {\em NeurIPS}, pages 1024--1034, 2017.

\bibitem{hay2009accurate}
M.~Hay, C.~Li, G.~Miklau, and D.~Jensen.
\newblock Accurate estimation of the degree distribution of private networks.
\newblock In {\em ICDM}, pages 169--178. IEEE, 2009.

\bibitem{he2021stealing}
X.~He, J.~Jia, M.~Backes, N.~Z. Gong, and Y.~Zhang.
\newblock Stealing links from graph neural networks.
\newblock In {\em USENIX Security}, 2021.

\bibitem{he2020stealing}
X.~He, J.-Y. Jia, M.~Backes, N.~Gong, and Y.~Zhang.
\newblock Stealing links from graph neural networks.
\newblock In {\em USENIX Security}, 2021.

\bibitem{hsieh2021netfense}
I.-C. Hsieh and C.-T. Li.
\newblock Netfense: Adversarial defenses against privacy attacks on neural
  networks for graph data.
\newblock {\em TKDE}, 2021.

\bibitem{jin2021universal}
D.~Jin, Z.~Yu, C.~Huo, R.~Wang, X.~Wang, D.~He, and J.~Han.
\newblock Universal graph convolutional networks.
\newblock {\em NeurIPS}, 34, 2021.

\bibitem{jin2020adversarial}
W.~Jin, Y.~Li, H.~Xu, Y.~Wang, and J.~Tang.
\newblock Adversarial attacks and defenses on graphs: A review and empirical
  study.
\newblock {\em arXiv preprint arXiv:2003.00653}, 2020.

\bibitem{kipf2016variational}
T.~N. Kipf and M.~Welling.
\newblock Variational graph auto-encoders.
\newblock {\em Bayesian Deep Learning Workshop (NeurIPS 2016)}, 2016.

\bibitem{kipf2016semi}
T.~N. Kipf and M.~Welling.
\newblock Semi-supervised classification with graph convolutional networks.
\newblock {\em ICLR}, 2017.

\bibitem{li2021adversarial}
J.~Li, T.~Xie, C.~Liang, F.~Xie, X.~He, and Z.~Zheng.
\newblock Adversarial attack on large scale graph.
\newblock {\em TKDE}, 2021.

\bibitem{lim2021large}
D.~Lim, F.~Hohne, X.~Li, S.~L. Huang, V.~Gupta, O.~Bhalerao, and S.~N. Lim.
\newblock Large scale learning on non-homophilous graphs: New benchmarks and
  strong simple methods.
\newblock {\em NeurIPS}, 34, 2021.

\bibitem{liu2018zeroth}
S.~Liu, X.~Li, P.-Y. Chen, J.~Haupt, and L.~Amini.
\newblock Zeroth-order stochastic projected gradient descent for nonconvex
  optimization.
\newblock In {\em GlobalSIP}, pages 1179--1183. IEEE, 2018.

\bibitem{liu2019unified}
X.~Liu, S.~Si, X.~Zhu, Y.~Li, and C.-J. Hsieh.
\newblock A unified framework for data poisoning attack to graph-based
  semi-supervised learning.
\newblock {\em NeurIPS}, 2019.

\bibitem{ma2020black}
J.~Ma, S.~Ding, and Q.~Mei.
\newblock Towards more practical adversarial attacks on graph neural networks.
\newblock {\em NeurIPS}, 33:4756--4766, 2020.

\bibitem{ma2019attacking}
Y.~Ma, S.~Wang, T.~Derr, L.~Wu, and J.~Tang.
\newblock Graph adversarial attack via rewiring.
\newblock page 1161–1169, 2021.

\bibitem{mnih2015human}
V.~Mnih, K.~Kavukcuoglu, D.~Silver, A.~A. Rusu, J.~Veness, M.~G. Bellemare,
  A.~Graves, M.~Riedmiller, A.~K. Fidjeland, G.~Ostrovski, et~al.
\newblock Human-level control through deep reinforcement learning.
\newblock {\em Nature}, 518(7540):529--533, 2015.

\bibitem{mu2021hard}
J.~Mu, B.~Wang, Q.~Li, K.~Sun, M.~Xu, and Z.~Liu.
\newblock A hard label black-box adversarial attack against graph neural
  networks.
\newblock In {\em CCS}, pages 108--125, 2021.

\bibitem{ribeiro2017struc2vec}
L.~F. Ribeiro, P.~H. Saverese, and D.~R. Figueiredo.
\newblock struc2vec: Learning node representations from structural identity.
\newblock In {\em SIGKDD}, pages 385--394, 2017.

\bibitem{riesen2008iam}
K.~Riesen and H.~Bunke.
\newblock Iam graph database repository for graph based pattern recognition and
  machine learning.
\newblock In {\em Joint IAPR International Workshops on SPR and SSPR}, pages
  287--297. Springer, 2008.

\bibitem{rigaki2020survey}
M.~Rigaki and S.~Garcia.
\newblock A survey of privacy attacks in machine learning.
\newblock {\em arXiv preprint arXiv:2007.07646}, 2020.

\bibitem{sen2008collective}
P.~Sen, G.~Namata, M.~Bilgic, L.~Getoor, B.~Galligher, and T.~Eliassi-Rad.
\newblock Collective classification in network data.
\newblock {\em AI magazine}, 29(3):93--93, 2008.

\bibitem{shamir2017optimal}
O.~Shamir.
\newblock An optimal algorithm for bandit and zero-order convex optimization
  with two-point feedback.
\newblock {\em The Journal of Machine Learning Research}, 18(1):1703--1713,
  2017.

\bibitem{shervashidze2011weisfeiler}
N.~Shervashidze, P.~Schweitzer, E.~J. Van~Leeuwen, K.~Mehlhorn, and K.~M.
  Borgwardt.
\newblock Weisfeiler-lehman graph kernels.
\newblock {\em Journal of Machine Learning Research}, 12(9), 2011.

\bibitem{shokri2017membership}
R.~Shokri, M.~Stronati, C.~Song, and V.~Shmatikov.
\newblock Membership inference attacks against machine learning models.
\newblock In {\em 2017 IEEE Symposium on Security and Privacy}, pages 3--18.
  IEEE, 2017.

\bibitem{sun2018adversarial}
L.~Sun, Y.~Dou, C.~Yang, J.~Wang, P.~S. Yu, L.~He, and B.~Li.
\newblock Adversarial attack and defense on graph data: A survey.
\newblock {\em arXiv preprint arXiv:1812.10528}.

\bibitem{sun2020non}
Y.~Sun, S.~Wang, X.~Tang, T.-Y. Hsieh, and V.~Honavar.
\newblock Non-target-specific node injection attacks on graph neural networks:
  A hierarchical reinforcement learning approach.
\newblock In {\em WWW}, 2020.

\bibitem{tramer2016stealing}
F.~Tram{\`e}r, F.~Zhang, A.~Juels, M.~K. Reiter, and T.~Ristenpart.
\newblock Stealing machine learning models via prediction apis.
\newblock In {\em USENIX}, pages 601--618, 2016.

\bibitem{velivckovic2017graph}
P.~Veli{\v{c}}kovi{\'c}, G.~Cucurull, A.~Casanova, A.~Romero, P.~Lio, and
  Y.~Bengio.
\newblock Graph attention networks.
\newblock {\em ICLR}, 2018.

\bibitem{wang2019attacking}
B.~Wang and N.~Z. Gong.
\newblock Attacking graph-based classification via manipulating the graph
  structure.
\newblock In {\em SIGSAC}, 2019.

\bibitem{wu2020model}
B.~Wu, X.~Yang, S.~Pan, and X.~Yuan.
\newblock Model extraction attacks on graph neural networks: Taxonomy and
  realization.
\newblock {\em arXiv preprint arXiv:2010.12751}, 2020.

\bibitem{wu2021linkteller}
F.~Wu, Y.~Long, C.~Zhang, and B.~Li.
\newblock Linkteller: Recovering private edges from graph neural networks via
  influence analysis.
\newblock {\em S\&P}, 2021.

\bibitem{wu2019adversarial}
H.~Wu, C.~Wang, Y.~Tyshetskiy, A.~Docherty, K.~Lu, and L.~Zhu.
\newblock Adversarial examples on graph data: Deep insights into attack and
  defense.
\newblock {\em IJCAI}, 2019.

\bibitem{wu2016methodology}
X.~Wu, M.~Fredrikson, S.~Jha, and J.~F. Naughton.
\newblock A methodology for formalizing model-inversion attacks.
\newblock In {\em 2016 IEEE 29th CSF}, pages 355--370. IEEE, 2016.

\bibitem{xu2021robustness}
J.~Xu, J.~Chen, S.~You, Z.~Xiao, Y.~Yang, and J.~Lu.
\newblock Robustness of deep learning models on graphs: A survey.
\newblock {\em AI Open}, 2:69--78, 2021.

\bibitem{xu2019topology}
K.~Xu, H.~Chen, S.~Liu, P.-Y. Chen, T.-W. Weng, M.~Hong, and X.~Lin.
\newblock Topology attack and defense for graph neural networks: An
  optimization perspective.
\newblock {\em IJCAI}, 2019.

\bibitem{yeom2018privacy}
S.~Yeom, I.~Giacomelli, M.~Fredrikson, and S.~Jha.
\newblock Privacy risk in machine learning: Analyzing the connection to
  overfitting.
\newblock In {\em 2018 IEEE 31st CSF}, pages 268--282. IEEE, 2018.

\bibitem{zhang2020personalized}
M.~Zhang, S.~Wu, M.~Gao, X.~Jiang, K.~Xu, and L.~Wang.
\newblock Personalized graph neural networks with attention mechanism for
  session-aware recommendation.
\newblock {\em TKDE}, 2020.

\bibitem{multiview}
X.~Zhang, L.~Zhang, B.~Jin, and X.~Lu.
\newblock A multi-view confidence-calibrated framework for fair and stable
  graph representation learning.
\newblock In {\em 2021 IEEE International Conference on Data Mining (ICDM)},
  2021.

\bibitem{zhang2020secret}
Y.~Zhang, R.~Jia, H.~Pei, W.~Wang, B.~Li, and D.~Song.
\newblock The secret revealer: generative model-inversion attacks against deep
  neural networks.
\newblock In {\em CVPR}, pages 253--261, 2020.

\bibitem{zhang2021inference}
Z.~Zhang, M.~Chen, M.~Backes, Y.~Shen, and Y.~Zhang.
\newblock Inference attacks against graph neural networks.
\newblock In {\em USENIX Security}, volume 2022, page~13, 2021.

\bibitem{zhang2020deep}
Z.~Zhang, P.~Cui, and W.~Zhu.
\newblock Deep learning on graphs: A survey.
\newblock {\em TKDE}, 2020.

\bibitem{ijcai2021-516}
Z.~Zhang, Q.~Liu, Z.~Huang, H.~Wang, C.~Lu, C.~Liu, and E.~Chen.
\newblock Graphmi: Extracting private graph data from graph neural networks.
\newblock In {\em IJCAI}, pages 3749--3755, 2021.

\bibitem{zhang2021motif}
Z.~Zhang, Q.~Liu, H.~Wang, C.~Lu, and C.-K. Lee.
\newblock Motif-based graph self-supervised learning for molecular property
  prediction.
\newblock {\em NeurIPS}, 34, 2021.

\bibitem{zugner2018adversarial}
D.~Z{\"u}gner, A.~Akbarnejad, and S.~G{\"u}nnemann.
\newblock Adversarial attacks on neural networks for graph data.
\newblock In {\em SIGKDD}, pages 2847--2856, 2018.

\bibitem{zugner2019adversarial}
D.~Z{\"u}gner and S.~G{\"u}nnemann.
\newblock Adversarial attacks on graph neural networks via meta learning.
\newblock {\em ICLR}, 2019.

\end{thebibliography}
\vspace{-0.5cm}

\ifCLASSOPTIONcaptionsoff
  \newpage
\fi



%

%

\begin{IEEEbiography}[{\includegraphics[width=1in,height=1.25in,clip,keepaspectratio]{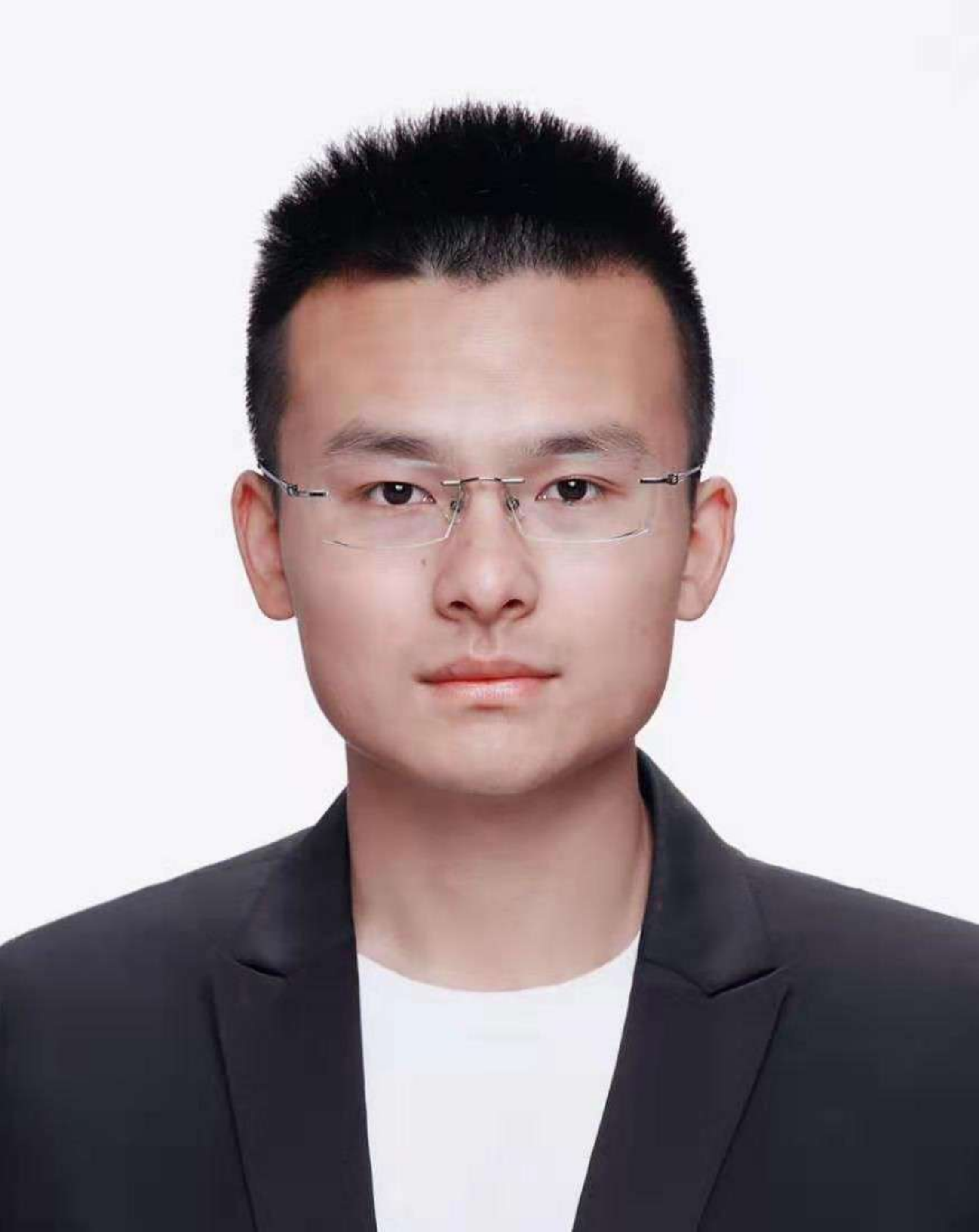}}]
	{Zaixi Zhang} received the BS degree from the department of gifted young, University of Science and Technology of China (USTC), China, in 2019. He is currently working toward the Ph.D. degree in the School of Computer Science and Technology at USTC. His main research interests include data mining, machine learning security and privacy, graph representation learning. He has published papers in referred conference proceedings, such as IJCAI, NeurIPS, KDD, and AAAI.
\end{IEEEbiography}

\begin{IEEEbiography}[{\includegraphics[width=1in,height=1.25in,clip,keepaspectratio]{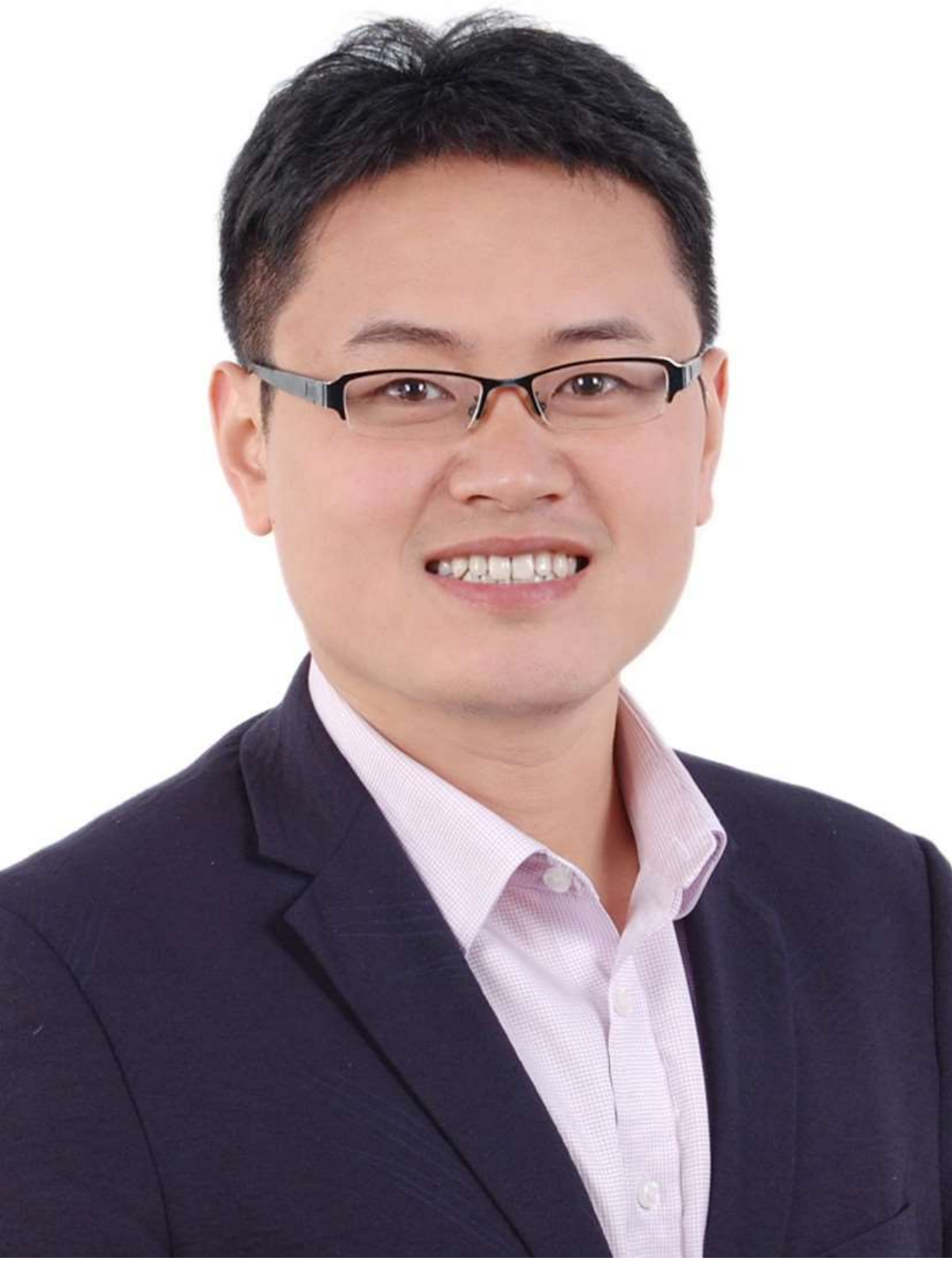}}]
	{Qi Liu} received the Ph.D. degree from University of Science and Technology of China (USTC), Hefei, China, in 2013. He is currently a Professor in the School of Computer Science and Technology at USTC. His general area of research is data mining and knowledge discovery. He has published prolifically in refereed journals and conference proceedings (e.g., TKDE, TOIS, KDD). He is an Associate Editor of IEEE TBD and Neurocomputing. He was the recipient of KDD' 18 Best Student Paper Award and ICDM'~11 Best Research Paper Award. He is a member of the Alibaba DAMO Academy Young Fellow. He was also the recipient of China Outstanding Youth Science Foundation in 2019. 
\end{IEEEbiography}
\vspace{-0.25cm}
\begin{IEEEbiography}[{\includegraphics[width=1in,height=1.25in,clip,keepaspectratio]{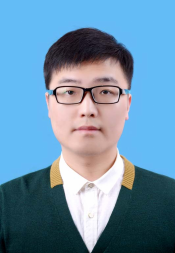}}]
{Zhenya Huang} recieved the B.E. degree from
Shandong University, in 2014 and the Ph.D. degree from University of Science and Technology
of China (USTC), in 2020. He is currently an
associate researcher of the School of Computer
Science and Technology, USTC. His main research interests include data mining, knowledge
discovery, representation learning and intelligent
tutoring systems. He has published more than
20 papers in refereed journals and conference
proceedings including TKDE, TOIS, KDD, AAAI.
\end{IEEEbiography}
\vspace{-0.25cm}
\begin{IEEEbiography}[{\includegraphics[width=1in,height=1.25in,clip,keepaspectratio]{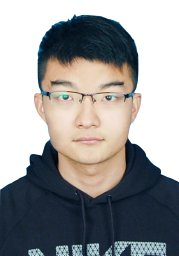}}]
{Hao Wang} Hao Wang is currently an associate researcher of the School of Computer
Science and Technology, USTC. His main research interests include data mining, representation learning, network embedding and recommender systems. He has published several papers in referred conference proceedings, such as TKDE, TOIS, NeurIPS, and KDD. 
\end{IEEEbiography}
\vspace{-0.25cm}
\begin{IEEEbiography}[{\includegraphics[width=1in,height=1.25in,clip,keepaspectratio]{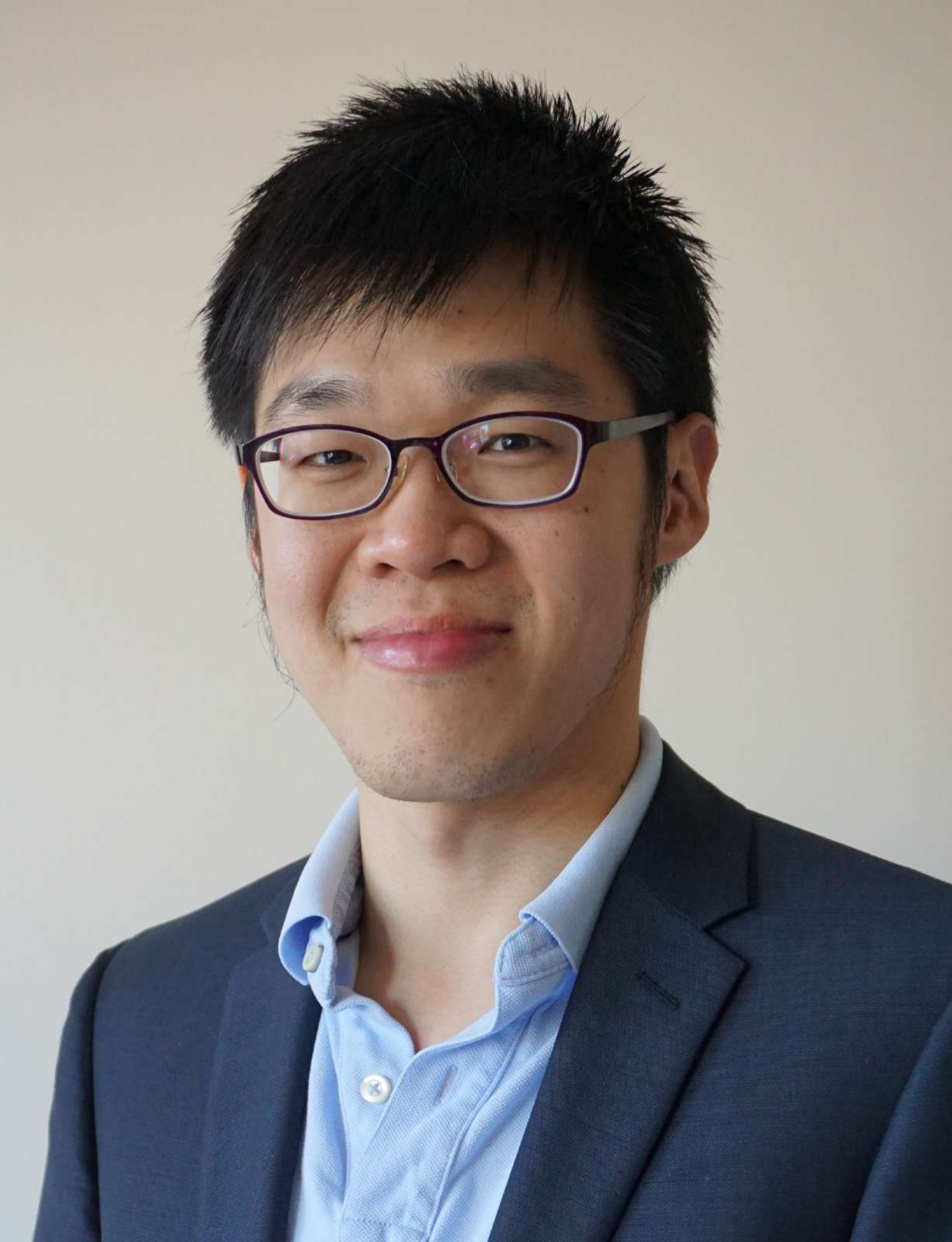}}]
	{Chee-Kong Lee} is a senior researcher at Tencent Quantum Lab. He obtained his undergraduate degree in physics from the National University of Singapore and PhD degree in theoretical chemistry from Massachusetts Institute of Technology (MIT). His research interests include quantum computing, computational chemistry, and applications of machine learning in chemistry and physics.
\end{IEEEbiography}
\vspace{-0.25cm}
\begin{IEEEbiography}[{\includegraphics[width=1in,height=1.25in,clip,keepaspectratio]{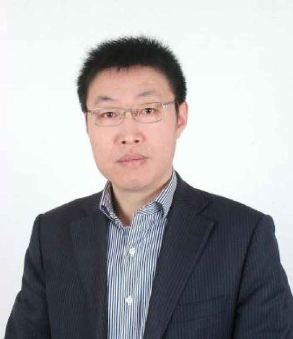}}]
	{Enhong Chen}(SM'07) is a professor and vice dean of the School of Computer Science at University of Science and Technology of China  (USTC). He received the Ph.D. degree from USTC. His general area of research includes data mining and machine learning, social network analysis and recommender systems. He has published more than 100 papers in refereed conferences and journals, including IEEE Trans. KDE, IEEE Trans. MC, KDD, ICDM, NIPS, and CIKM. He was on program committees of numerous conferences including KDD, ICDM, SDM. He received the Best Application Paper Award on KDD-2008, the Best Student Paper Award on KDD-2018 (Research), the Best Research Paper Award on ICDM-2011 and Best of SDM-2015. His research is supported by the National Science Foundation for Distinguished Young Scholars of China.
	He is a senior member of the IEEE.
\end{IEEEbiography}




\appendices
\clearpage
\section{Proof of Theorem 1}
\textbf{\emph{Theorem 1.}}
$Adv(e_i) \le Adv(e_j),~ \forall e_i, e_j \in \mathcal{E} \cap \mathcal{I}(e_i) \le \mathcal{I}(e_j)$ i.e. The adversary advantage is greater for edges with greater influence.

\noindent
The proof is based on the lemma \ref{lemma} from \cite{wu2016methodology}.
\begin{Lemma}
\label{lemma}
Suppose the target model $f$ is trained on data distribution $p(\mathcal{X},\mathcal{Y})$. $\mathcal{X}= (x_s, x_{ns})$, where $x_s$ and $x_{ns}$ denote the sensitive and non-sensitive part of feature respectively.  The optimal adversary advantage is
\[P_{\mathcal{X} \sim p(\mathcal{X},\mathcal{Y})}[f(x_s=1, x_{ns}) \neq f(x_s=0, x_{ns})] .\]
\end{Lemma}
\vspace{-0.3cm}
\textbf{\emph{Proof.}} In our graph setting, $x_{ns}$ refers to feature matrix $X$ and $x_s$ refers to edges.
Set $P[f^i_{\theta}(A, X)= y_i]$ = $P_{f_{\theta}}(y_i~ |~ A, X) = p$; $P[f^i_{\theta}(A_{-e}, X)= y_i] = P_{f_{\theta}}(y_i ~|~ A_{-e}, X) = q$.
For $f_\theta$, the prediction accuracy is higher with edge $e$, i.e. $q\le p$.
The adversary advantage is $Adv =[p(1-q)+q(1-p)]$.
Through variable substitution ($x=p-q, y=p+q$), we have $Adv= (y+\frac{x^2-y^2}{2}) $ and $\frac{\partial Adv}{\partial \mathcal{I}(e)} = \frac{\partial Adv}{\partial x}= p-q \ge 0$. Hence, $Adv(e_i) \le Adv(e_j)$ if $e_i, e_j \in \mathcal{E}$ and $\mathcal{I}(e_i) \le \mathcal{I}(e_j)$.
\hfill $\square$
\vspace{-0.3cm}
\section{Proof of Theorem 2}
\begin{Lemma}
Define $\mathcal{L}_\mu(a) = \mathbb{E}_{v\in U_b}[\mathcal{L}_{attack}(a+\mu v)]$, where $U_b$ denotes a uniform distribution with respect to the unit Euclidean ball. We have:
\begin{equation}
    \mathbb{E}[\hat{\nabla}\mathcal{L}_{attack}(a)] =\mathbb{E}_u [d/\mu \mathcal{L}_{attack}(a+\mu u)] =\nabla \mathcal{L}_\mu(a),
\end{equation}
where $u$ is a vector picked uniformly at random from the Euclidean unit sphere. Under Assumption 1 and 2, $\mathcal{L}_\mu$ is $L_1$-Lipschitz continuous and has $L_2$-Lipschitz continuous gradient. 
\end{Lemma}
\textbf{\emph{Proof.}}
The proof can be found in \cite{flaxman2004online} Lemma 1.

\begin{Lemma}
If Assumption 2. holds and $\eta_k \in (0, 1/L_2)$, then the outputs $\{a_t\}_{t=0}^{T-1}$ from the gradient estimation algorithm satisfies:
\begin{align}
    &\sum_{t=0}^{T-1}((2\eta_t - L_2 \eta_t^2) \mathbb{E}[\|P_{[0,1]}(a_r, \hat{\nabla} \mathcal{L}_{attack}(a_t), \eta_t)\|^2])\\
    &\le \sum_{t=0}^{T-1} (2\eta_t\mathbb{E}[\|\hat{\nabla}\mathcal{L}_{attack}(a_t) - \mathbb{E}[\hat{\nabla}\mathcal{L}_{attack}(a_t)|a_t]\|^2])\\
    &+ 2\mu^2L_2 + c_2,
\end{align}
where $\mathbb{E}$ is taken with respect to all the randomness (e.g., random directions), $P_{[0,1]}$ is the gradient projection given by Equ. \ref{projection}, and $c_2 = 2(f(a_0)-f(a^*))$ ($a^*$ is the globally optimal solution).
\end{Lemma}

\textbf{\emph{Proof.}}
The proof can be found in \cite{liu2018zeroth} Proposition 1.

\begin{theorem}
Suppose Assumption 1$\&$2 hold. If we randomly pick $a_r$, whose dimension is $d$, from $\{a_t\}_{t=0}^{T-1}$ with probability $P(r=t) = \frac{2\eta_t-L_2\eta_t^2}{\sum_{t=0}^{T-1}2\eta_t-L_2\eta_t^2}$ Then, we have the following bound on the convergence rate with $\eta_t = \mathcal{O}(\frac{1}{\sqrt{T}})$ and $\mu = \mathcal{O}(\frac{1}{\sqrt{dq}})$: 
\begin{equation}
    \mathbb{E} [\|P_{[0,1]}(a_r, \hat{\nabla} \mathcal{L}_{attack}(a_r), \eta_r)\|^2] = \mathcal{O}(\frac{1}{\sqrt{T}}+\frac{d+q}{q}),
\end{equation}
where $P_{[0,1]}(a_r, \hat{\nabla} \mathcal{L}_{attack}(a_r), \eta_r) \coloneqq (1/\eta_r) [a_r - P_{[0,1]}(a_r - \eta_r \hat{\nabla} \mathcal{L}_{attack}(a_r))$ is the rectified gradient.
\end{theorem}
\textbf{\emph{Proof.}}
Since $\mathbb{E}[\hat{\nabla}\mathcal{L}_{attack}(a_t)|a_t] = \nabla\mathcal{L}_\mu(a_t)$, we have:
\begin{align}
    &\mathbb{E}[\|\hat{\nabla}\mathcal{L}_{attack}(a_t) - \mathbb{E}[\hat{\nabla}\mathcal{L}_{attack}(a_t)|a_t]\|^2] \\
    &= \mathbb{E}[\|\hat{\nabla}\mathcal{L}_{attack}(a_t) - \nabla\mathcal{L}_\mu(a_t)\|^2] \\
    &\le \mathbb{E}[\|\hat{\nabla}\mathcal{L}_{attack}(a_t)\|^2] \\
    &\le \frac{(c_1 d + 4q)L_1^2}{4q},
\end{align}
where the final inequality comes from \cite{liu2018zeroth} Lemma 1. Based on Lemma 3, we have:
\begin{align}
    &\sum_{t=0}^{T-1}((2\eta_t - L_2 \eta_t^2) \mathbb{E}[\|P_{[0,1]}(a_r, \hat{\nabla} \mathcal{L}_{attack}(a_t), \eta_t)\|^2])\\
    &\le \frac{(c_1 d + 4q)L_1^2}{2q}\sum_{t=0}^{T-1}\eta_t + 2\mu^2L_2 + c_2.
\end{align}
If we sample $a_r$ from $\{a_t\}_{t=0}^{T-1}$ with the probability in Theorem 3, we can then obtain:
\begin{align}
    &\mathbb{E} [\|P_{[0,1]}(a_r, \hat{\nabla} \mathcal{L}_{attack}(a_r), \eta_r)\|^2]\\
    & \le \frac{(c_1 d + 4q)L_1^2 \sum_{t=0}^{T-1}\eta_t}{2q\sum_{t=0}^{T-1}((2\eta_t - L_2 \eta_t^2)}+\frac{2\mu^2L_2 + c_2}{\sum_{t=0}^{T-1}((2\eta_t - L_2 \eta_t^2)}.
\end{align}
If we choose $\eta_t = \frac{c_\eta}{\sqrt{T}}\in (0, 1/L_2)$, then the above inequality implies the convergence rate of $\mathcal{O}(\frac{1}{\sqrt{T}}+\frac{d+q}{q})$. 
\hfill $\square$

\section{Macro-level Graph Statistics}
\begin{itemize}
    \item \textbf{Degree Distribution} The degree distribution $P(k)$ of a
graph is defined to be the fraction of nodes in the graph
with degree $k$. It is widely used to characterize the graph statistics.
    \item \textbf{Local Clustering Coefficient (LCC).}
    The LCC of a node
measures the closeness to its neighbors to form a cluster. It is primarily proposed to determine if a graph is a small-world network.
    \item \textbf{Betweenness Centrality (BC).} Betweenness centrality measures the node centralities based on the shortest paths. For every pair of nodes in a graph, there exists at
least one shortest path among all possible paths, which contains the minimal
the number of edges. The betweenness centrality for a certain node is defined as the number of these shortest paths that pass through the node.
    \item \textbf{Closeness Centrality (CC).} Closeness Centrality is another measure
of node centrality. It is defined as the reciprocal of the sum of the length of the shortest paths between
the node and all other nodes in the graph. Generally, the
larger the CC of a node is, the closer it is to all other nodes.
\end{itemize}

\end{document}